\documentclass{article} % For LaTeX2e
\usepackage{iclr2026_conference, times}
\iclrfinalcopy
\usepackage{microtype}
\usepackage{booktabs}
\usepackage{graphicx}
\usepackage{booktabs}
\usepackage{amsmath}
\usepackage{amssymb}
\usepackage{mathtools}
\usepackage{amsthm}
\usepackage{makecell}
\usepackage{bm}
\usepackage{tikz}
\usepackage{caption}   % 提供 \captionof
\usepackage{placeins}
\theoremstyle{plain}
\newtheorem{theorem}{Theorem}[section]

\theoremstyle{definition}

\theoremstyle{remark}

\def\ie{\emph{i.e.}} \def\eg{\emph{e.g.}}

\newcommand{\para}[1]{\vspace{.05in}
\noindent
\textbf{#1}\quad}

\newcommand*{\circled}[1]{\tikz[baseline=(char.base)]{%
\node[shape=circle,fill,inner sep=.5pt](char){%
\textcolor{white}{\footnotesize #1}};}}
\usepackage{amsmath,amsfonts,bm}

\def\eqref#1{equation~\ref{#1}}
\def\1{\bm{1}}

\DeclareMathAlphabet{\mathsfit}{\encodingdefault}{\sfdefault}{m}{sl}
\SetMathAlphabet{\mathsfit}{bold}{\encodingdefault}{\sfdefault}{bx}{n}

\usepackage{hyperref}
\usepackage{url}

\AtEndPreamble{
    \usepackage[capitalize]{cleveref}
    \crefformat{section}{\S#2#1#3}
    \Crefformat{section}{\S#2#1#3}
    \crefname{figure}{Fig.}{Figs.}
    \Crefname{figure}{Fig.}{Figs.}
    \crefname{table}{Tab.}{Tabs.}
    \Crefname{table}{Tab.}{Tabs.}
}

\title{CERSA: Cumulative Energy-Retaining Subspace Adaptation for Memory-Efficient
Fine-Tuning}

\author{
Jingze Ge$^{1}$ \quad
Xue Geng$^{3}$ \quad
Yun Liu$^{2}$ \quad
Wanqi Dong$^{1}$ \\
Wang Zhe Mark$^{3}$ \quad
Min Wu$^{3}$ \quad
Ngai-Man Cheung$^{4}$ \quad
Bharadwaj Veeravalli$^{1}$ \quad
Xulei Yang$^{3}$ \\
\\
$^{1}$National University of Singapore \\
$^{2}$Nankai University \\
$^{3}$Institute for Infocomm Research (I$^2$R), A*STAR \\
$^{4}$Singapore University of Technology and Design \\
\\
\texttt{jingze.ge@u.nus.edu} \quad
\texttt{geng\_xue@i2r.a-star.edu.sg} \\
\texttt{liuyun@mail.nankai.edu.cn} \quad
\texttt{wanqi.dong@u.nus.edu} \\
\texttt{wumin@a-star.edu.sg} \quad
\texttt{ngaiman\_cheung@sutd.edu.sg} \\
\texttt{elebv@nus.edu.sg} \quad
\texttt{YANG\_XULEI@I2R.A-STAR.EDU.SG} \\
}

\begin{document}
    \maketitle

    \begin{abstract}
    To mitigate the memory constraints associated with fine-tuning large pre-trained
    models, existing parameter-efficient fine-tuning (PEFT) methods, such as LoRA,
    rely on low-rank updates. However, such updates fail to fully capture the
    rank characteristics of the weight modifications observed in full-parameter
    fine-tuning, resulting in a performance gap. Furthermore, LoRA and other existing
    PEFT methods still require substantial memory to store the full set of frozen
    weights, limiting their efficiency in resource-constrained settings. To address
    these limitations, we introduce \textbf{Cumulative Energy-Retaining Subspace
    Adaptation (CERSA)}, a novel fine-tuning paradigm that leverages singular value
    decomposition (SVD) to retain only the principal components responsible for 90\%
    to 95\% of the spectral energy. By fine-tuning low-rank representations
    derived from this principal subspace, CERSA significantly reduces memory consumption.
    We conduct extensive evaluations of CERSA across models of varying scales
    and domains, including image recognition, text-to-image generation, and natural
    language understanding. Empirical results demonstrate that CERSA
    consistently outperforms state-of-the-art PEFT methods while achieving
    substantially lower memory requirements. The code will be released.
    \vspace{-.2in}
\end{abstract}

    \section{Introduction}
Fine-tuning pre-trained large models for specific tasks has become a common
practice to achieve superior performance in both natural language processing and
computer vision domains \citep{hu2022lora,sun2024svfit,meng2024pissa}. Pre-trained
models, which have been trained on extensive and diverse datasets \citep{deng2009imagenet,lin2014microsoft},
accumulate rich and general knowledge, enabling them to outperform models
trained from scratch. However, fine-tuning the entire pre-trained model
typically demands substantial computational resources like memory, particularly for
large-scale models based on transformer architectures, such as ViT-Large \citep{dosovitskiy2021image}
and DeBERTaV3 \citep{he2023debertav3}. Unlike massive training clusters, often
equipped with thousands of GPUs for pre-training, fine-tuning is more likely to occur
on consumer-grade GPUs to support diverse downstream applications. Consequently,
reducing the number of tunable parameters and the memory footprint has become a focal
point in parameter-efficient fine-tuning (PEFT) research \citep{hu2022lora,zi2023delta,zhang2023lora,kopiczko2024vera,gu2022ppt,ren2024mini,valipour2023dylora}.

Existing PEFT methods aim to fine-tune only a small subset of parameters within
pre-trained models \citep{rebuffi2017learning,li2021prefix,lester2021power},
which significantly reduces memory requirements. Since fewer parameters are updated during backpropagation, the demand for memory to store gradients and optimizer
states decreases. Among the most popular methods are LoRA \citep{hu2022lora} and
its variants \citep{zi2023delta,zhang2023lora,kopiczko2024vera,ren2024mini,sun2024svfit,meng2024pissa},
which introduce two low-rank matrices, $\bm{B}\in \mathbb{R}^{m \times r}$ and $\bm
{A}\in \mathbb{R}^{r \times n}$ ($r \ll m$, $r \ll n$), to reparameterize fine-tuning as $\bm{B}\times \bm{A}$. Here, the pre-trained weight matrix
% (or the residual matrix in PiSSA)
$\bm{W}\in \mathbb{R}^{m \times n}$ is frozen, and only the newly added low-rank matrices are trained.

Despite these advances, most existing methods focus on reducing memory usage by exploiting
the low-rank nature of gradients during training \citep{hu2022lora,zi2023delta,kopiczko2024vera,gu2022ppt,ren2024mini,valipour2023dylora}.
However, the full weight matrices must be stored in memory, with few approaches directly
compressing the pre-trained weights. As a result, the total memory consumption for
weights, gradients, and optimizer states often remains tied to the size of the pre-trained weights. Besides, SVFit \citep{sun2024svfit} and SVFT \citep{lingam2024svft} use
singular value decomposition (SVD) to compress pre-trained weights but require storing
two full singular vector matrices of size $\mathbb{R}^{n \times n}$, limiting
the memory savings despite their low trainable parameter counts (see \cref{sec:weight_decomp}).
Furthermore, recent studies \citep{shuttleworth2024lora} reveal a key limitation
of LoRA: it introduces intruder dimensions that degrade the model's performance
on learned tasks. These findings motivate us to directly preserve the major components of pre-trained weights, enabling memory-efficient fine-tuning while maintaining the prior knowledge encoded during pre-training.

To this end, we propose \textbf{Cumulative Energy-Retaining Subspace Adaptation (CERSA)}, a memory-efficient fine-tuning method for pre-trained weights. The key idea is to apply SVD to each weight matrix and truncate it to retain only the components that preserve most of the cumulative energy (typically 90\%–95\%). Since singular values of weight matrices follow a heavy-tailed distribution, a small set of dominant singular vectors suffices for adapting the model to downstream tasks. As illustrated in \cref{fig:layerwise_svd}, depending on the matrix position, retaining only 10\%–50\% of the original dimensions is often enough to capture the principal energy. This enables substantial memory savings during fine-tuning with minimal performance loss. For example, in ViT-Large \citep{dosovitskiy2021image}, keeping 95\% of the cumulative energy yields a memory footprint comparable to LoRA (rank=32) \citep{hu2022lora}, as shown in \cref{fig:memory-barchart}. Reducing the threshold to 90\% further lowers memory usage below that of the original pre-trained weights, while causing only a negligible drop—about 0.3\% on average across three image classification datasets (\cref{tab:cer-rate}). As illustrated in \cref{fig:perf-mem}, CERSA achieves a clearly superior accuracy-memory trade-off compared to baseline methods, making it especially effective under strict memory constraints.

The primary contributions of this paper are as follows:
% \vspace{-3mm}
\begin{itemize}
    \item We propose CERSA, a memory-efficient PEFT method that uses SVD to
        retain the primary cumulative energy of pre-trained model weights and
        fine-tunes within the principal subspace. This reduces memory usage below
        the weight size, improves fine-tuning efficiency compared to LoRA~\citep{hu2022lora}, and minimizes the forgetting of prior knowledge.

    \item We provide a theoretical analysis of CERSA, showing that fine-tuning
        within the principal cumulative energy subspace is sufficient for
        adapting the model to downstream tasks. This subspace overlaps significantly
        with those required for most tasks, helping retain pre-trained knowledge
        during fine-tuning.

    \item We comprehensively evaluate CERSA on image classification and natural language understanding tasks. Results demonstrate that CERSA consistently outperforms state-of-the-art PEFT baselines while achieving the best accuracy-memory trade-off, highlighting its effectiveness under constrained memory budgets.
\end{itemize}

\begin{figure*}[!t]
    \vspace{-.2in}
    \centering
    \begin{minipage}[t]{0.48\textwidth}
        \centering
        \includegraphics[width=0.95\linewidth]{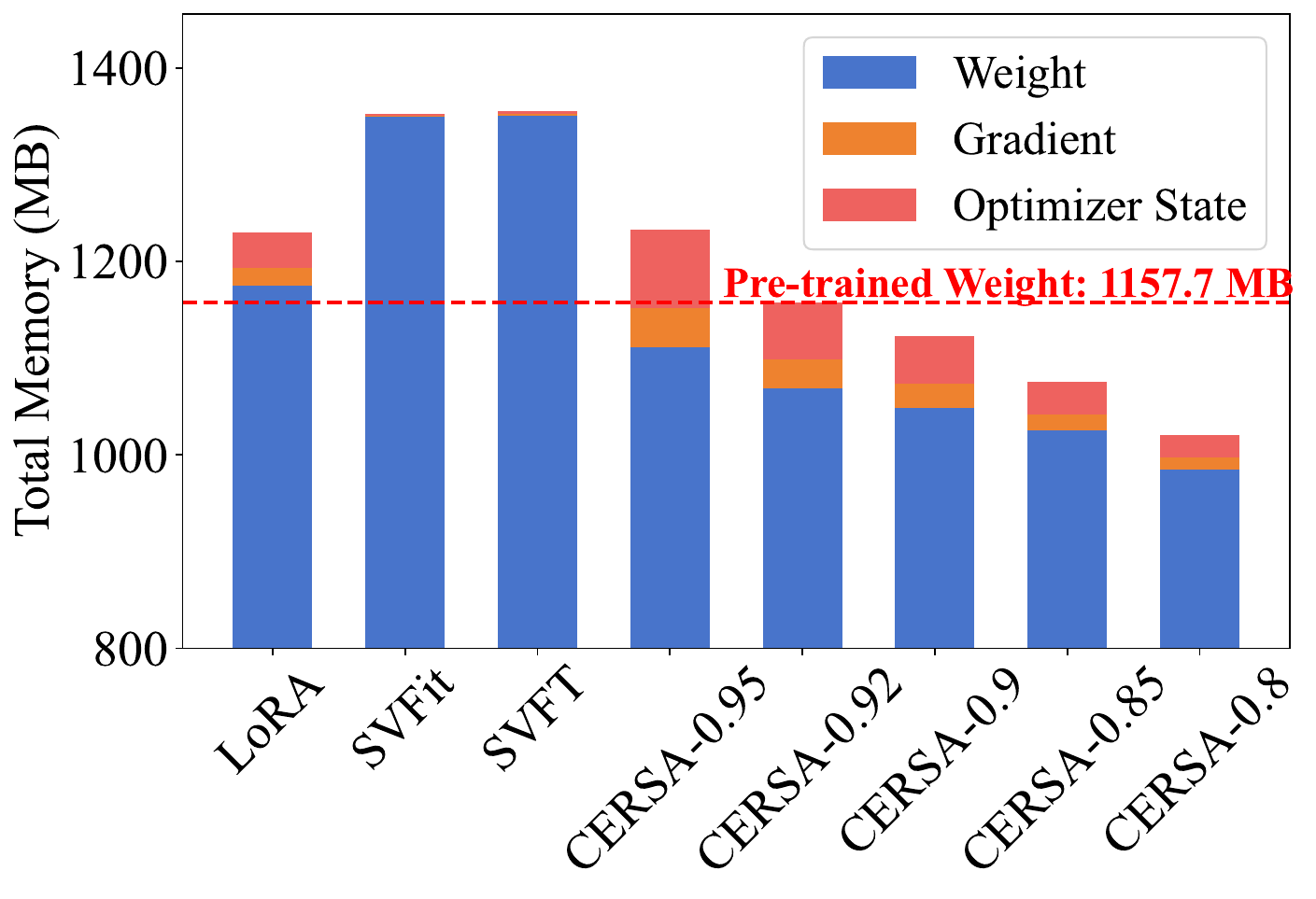}
        \vspace{-.15in}
        \captionof{figure}{Memory footprint comparison for fine-tuning ViT-Large~\citep{dosovitskiy2021image}.}% pre-trained on ImageNet-21K.}
        \label{fig:memory-barchart}
    \end{minipage}
    %—— 第一列：图 1 ——%
    \hfill
    %—— 第二列：先表1再表2 ——%
    \begin{minipage}[t]{0.50\textwidth}
        \centering
        \includegraphics[width=\linewidth]{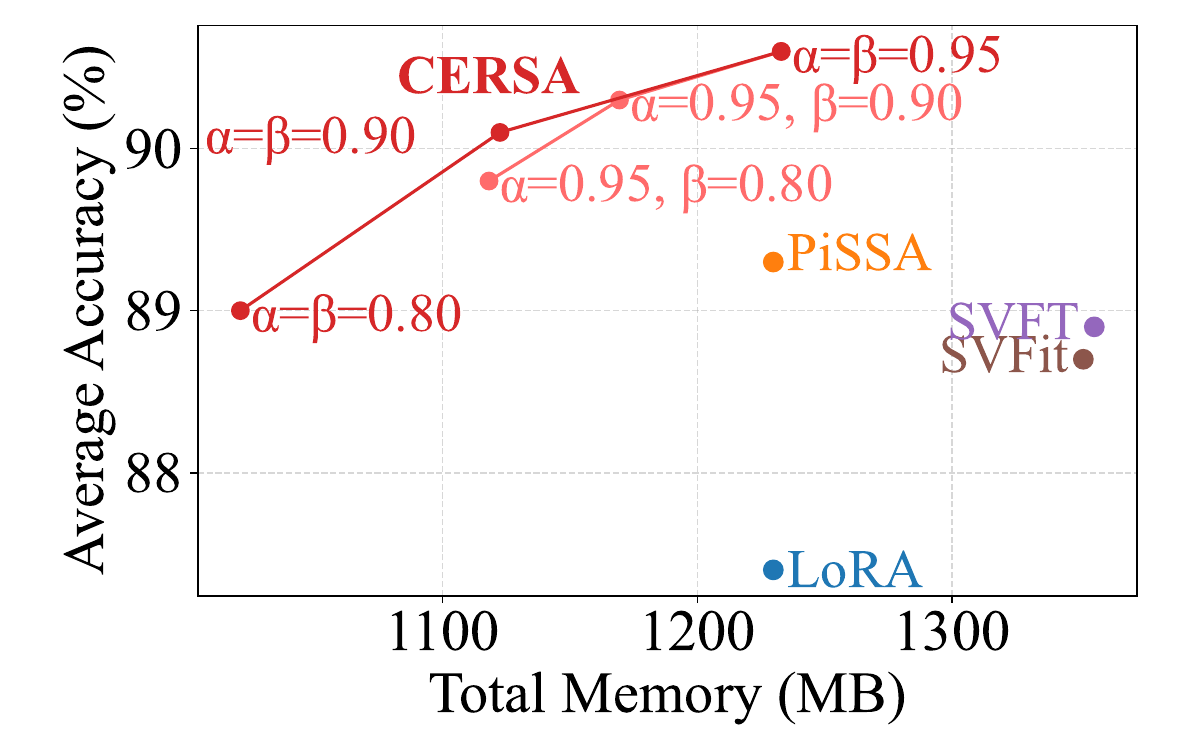}
        \vspace{-.3in}
        \captionof{figure}{Average accuracy(see~\cref{tab:vit_large_results}) versus total memory usage on ViT-Large~\citep{dosovitskiy2021image}} \label{fig:perf-mem}
    \end{minipage}
    \vspace{-.23in}
\end{figure*}

    \section{Related Work}
\subsection{Low-rank Adaptation}
LoRA~\citep{hu2022lora} is a key method in PEFT, reducing memory usage by
decomposing weight updates into low-rank matrices while keeping pre-trained
weights frozen. This enables the efficient fine-tuning of large models.
Enhancements to LoRA~\citep{hu2022lora} can be categorized into three types:
weight-driven, data-driven, and adaptive methods.

Weight-driven methods add adapters derived from weight decomposition on top of
frozen pre-trained weights, directly manipulating the weight space via matrix decompositions
and orthonormal constraints. Representative approaches, including PiSSA~\citep{meng2024pissa},
OLoRA~\citep{wang2023orthogonal}, MiLoRA~\citep{wang2024milora}, LoRA-XS~\citep{balazy2024lora},
and DoRA~\citep{liu2024dora}, introduce techniques such as SVD-based initialization,
QR-based orthonormal initialization, and minor singular component adaptation to enhance
representation learning and convergence speed.

Data-driven methods leverage model activations, gradients, or data distributions
to guide adapter updates. Techniques like LoRA-GA~\citep{wang2025lora}, LoRA-Pro~\citep{wang2024lorapro},
LaMDA~\citep{azizi2024lamda}, and EVA~\citep{paischer2024one} employ strategies such
as aligning low-rank gradients with full fine-tuning gradients and performing SVD
on mini-batch activations for variance-aware initialization, thereby improving
adaptation efficiency through data-informed adjustments.

Adaptive methods dynamically configure adapters by task characteristics or layer importance to optimize parameter utilization. Approaches such as AdaLoRA~\citep{zhang2023adaptive} and EVA~\citep{paischer2024one} employ rank allocation by layer importance and variance-aware adjustments, effectively balancing model capacity with computational cost to achieve efficient fine-tuning.

Despite these advancements, most LoRA-based methods store the entire frozen weight
matrix with multiple adapters, offering limited memory savings over the
original LoRA~\citep{hu2022lora}. This underscores the need for more efficient
methods to further reduce memory and computational costs.

\subsection{Weight-Decomposition-Based Method}
\label{sec:weight_decomp} To further minimize the number of parameters required for
fine-tuning and reduce computational costs, weight-decomposition-based methods
have been developed to process pre-trained weights. Generally, the basic step of
weight-decomposition-based methods~\citep{han2023svdiff} is to decompose the original
weight matrix $\bm{W}$ into $\bm{U}$, $\bm{\Sigma}$, and $\bm{V}$. SVFit~\citep{sun2024svfit}
fine-tunes only the top-$k$ singular values, freezing $\bm{U}$ and $\bm{V}$ to
retain principal components. SVFT~\citep{lingam2024svft} freezes $\bm{\Sigma}$
and introduces a sparse adapter for task-specific adaptation. SVDiff~\citep{han2023svdiff}
applies singular value fine-tuning to diffusion models, reducing storage while mitigating
overfitting. WeLore~\citep{jaiswal2024WeLore} optimizes rank reduction across layers
by identifying low-rank components for selective fine-tuning, enhancing
efficiency with minimal performance loss.

Although these methods reduce trainable parameters, they require storing $\bm{U}$
and $\bm{V}$, doubling the original weight size~\citep{lingam2024svft}. When
accounting for gradients and optimizer states, their memory footprint exceeds twice
that of the pre-trained weights, making them more memory-intensive than LoRA~\citep{hu2022lora}
and other PEFT methods.

    \section{Methodology}
Fine-tuning pre-trained models using Singular Value Decomposition (SVD) has
proven to be an effective approach for adapting large-scale models while
minimizing parameter updates~\citep{han2023svdiff,sun2024svfit,lingam2024svft}.
However, traditional SVD-based fine-tuning incurs substantial computational and memory
overhead by necessitating the storage of two full decomposed matrices,
effectively doubling memory consumption compared to standard weight storage. Moreover,
freezing the left and right singular matrices restricts the model’s
expressiveness, making it suboptimal relative to full-parameter fine-tuning. To
address these limitations, we propose a constrained optimization framework that selectively
updates the principal components using trainable matrices while discarding components
associated with minor singular vectors. By fine-tuning within the principal
subspace of the weight matrix, our method retains the core representational capacity
of the pre-trained model while significantly reducing memory requirements,
thereby enabling efficient and stable adaptation to downstream tasks.

\begin{figure}[!t]
    \vspace{-.2in}
    \centering
    \includegraphics[width=\linewidth]{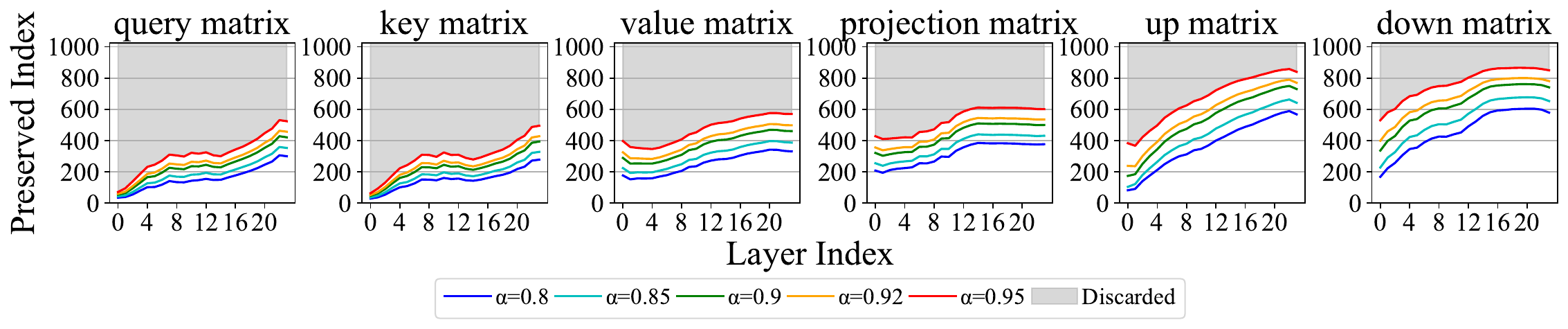}
    \vspace{-.30in}
    \caption{Preserved singular value indices in ViT-Large \citep{dosovitskiy2021image}
    (pre-trained on ImageNet-21K~\citep{deng2009imagenet}) across layers and
    weight matrices under different cumulative energy retention rates. The query
    ({\bf Q}), key ({\bf K}), value ({\bf V}), and projection ({\bf P}) matrices
    correspond to weight matrices in self-attention, while up ({\bf UP}) and down
    ({\bf DN}) matrices represent the weight matrices of the first and second linear
    operations of the multilayer perceptron (MLP), respectively.}
    \label{fig:layerwise_svd}
    \vspace{-.15in}
\end{figure}

\subsection{Layer-wise Rank Selection}
In existing methods, a persistent challenge is that, regardless of the attached adapter,
the original
% model weights must remain in memory.
pre-trained weights $\bm{W}\in \mathbb{R}^{m \times n}$ will impose a memory
cost of $\mathcal{O}(mn)$ and incur a computational overhead of
$\mathcal{O}(mn)$ during forward propagation, even for fully frozen matrices.
% The space complexity of \( \mathbb{R}^{m, n} \) not only requires storing \( mn \) elements but also incurs computational overhead of \( m^2n \) operations during forward propagation, even for fully frozen matrices.
As a result, no matter how parameter-efficient the fine-tuning method is, this
storage and computation burden remains unavoidable. Inspired
%by the principles and analysis of
by PiSSA~\citep{meng2024pissa}, we propose to retain the most significant cumulative
energy~\citep{jolliffe2002principal} of the weight matrix in terms of the
$\bm{U}$ and $\bm{V}$ matrices of its SVD, assuming that the subspace defined by
$\bm{U}$ and $\bm{V}$ matrices is sufficient for most fine-tuning scenarios.
%while the residual matrix contributes negligible energy and primarily represents high-rank components. The tail of the residual often consists of noise, with the remaining dimensions having limited relevance to downstream tasks. The spectral energy contained within the subspace defined by \( U \) and \( V \) matrices is sufficient for most fine-tuning scenarios.
% To reduce dimensionality while preserving as much pre-trained information as possible, we compress weight matrices using SVD.
% To achieve dimensionality reduction while preserving as much pre-trained information as possible, we compress weight matrices using SVD.
To further reduce the dimensionality, we propose using truncated SVD as it provides
the optimal low-rank approximation in terms of the Frobenius norm \citep{eckart1936approximation}.
%The Eckart–Young–Mirsky theorem~\citep{} establishes that truncation based on SVD provides the optimal low-rank approximation by minimizing the error in terms of the spectral norm or Frobenius norm.

Moreover, the singular value distributions of pre-trained weights vary significantly
across layers, influenced by both the layer depth and the type of weight matrix.
%Moreover, singular value distributions across different layers of pre-trained models exhibit substantial variation, influenced not only by the layer's position within the model but also by the type of matrix being decomposed.
To effectively extract the principal components across layers, we propose to
retain the cumulative energy of the truncated SVD using the \textit{cumulative
energy retention rate}~\citep{eckart1936approximation}.
% and further explore different settings.
% we analyze the truncation points of singular values at various fraction-of-variance~\citep{jolliffe2002principal} settings.
% This analysis underscores the importance of applying layer-wise SVD truncation to capture the unique characteristics of each weight matrix.
This rate measures the proportion of total cumulative energy retained in the selected
components after truncation and is calculated as:
\vspace{-2.0mm}
\begin{equation}
    \alpha = \frac{\sum_{i=1}^{k}s_{i}^{2}}{\sum_{j=1}^{N}s_{j}^{2}}, \label{eq:cumulative_energy}
    \vspace{-2.0mm}
\end{equation}
where $s_{i}$ represents the singular values corresponding to the $i$-th
principal component, $k$ denotes the number of selected singular values after truncation,
and $N$ is the total number of singular values. The numerator, $\sum_{i=1}^{k}s_{i}^{2}$, represents the energy retained in the first $k$ singular values, corresponding to the $k$ most significant components of the matrix. The denominator, $\sum_{j=1}^{N}s_{j}^{2}$, represents the total energy of the original matrix. Consequently, $\alpha$ quantifies the ratio of retained to total energy, reflecting the proportion of the matrix's variance preserved in the truncated representation.

\cref{eq:cumulative_energy} implies that higher singular values
%In the context of pre-trained model weights, this relationship implies that higher singular values
% (those with larger \(s_i^2\))
contribute more to the total cumulative energy of the matrix. By setting a specific
cumulative energy retention rate (\eg, $\alpha = 0.95$) across different layers,
one can determine the minimum rank $k$ required to retain a desired proportion of
the cumulative energy of the weight matrix. Hence, the cumulative energy retention
rate enables us to balance dimensionality reduction and information preservation.
% This layer-wise approach ensures that the truncation adapts to the unique singular value distribution of each weight matrix, allowing effective dimensionality reduction with minimal information loss.
%while minimizing the loss of important information.

% This method leverages the fact that spectral energy, as represented by the squared singular values, directly correlates with the variance captured by the corresponding components in the matrix. Hence, the cumulative energy retention rate provides a principled way to balance dimensionality reduction and information preservation.

\cref{fig:layerwise_svd} illustrates the preserved singular value indices in ViT-Large
\citep{dosovitskiy2021image}, pre-trained on ImageNet-21K~\citep{deng2009imagenet},
after SVD decomposition. The rank values are computed across different layer types
at various cumulative energy retention rates \{0.8, 0.85, 0.9, 0.92, 0.95\}. As
observed, the indices of the preserved singular value arrays exhibit an increasing
trend from the lower to upper layers, indicating that lower layers allow for
greater compression. Additionally, compared to the MLP layers, the query, key, value,
and projection matrices in the self-attention module have lower cutoff index
values at the same cumulative energy retention rate. These observations underscore
the importance of layer-wise rank selection in optimizing model efficiency.

\subsection{Trainable Matrix in the Principal Subspace}
By computing the cumulative energy, we establish the criterion for top-$k$ truncation of weight matrices. Using the truncation index in \cref{fig:layerwise_svd}, we maximize the removal of residual ranks layer-wise, thereby optimizing memory usage. The retained ranks determined by the chosen cumulative energy threshold represent the trade-off between performance and memory budget.

\begin{figure*}[!t]
    \centering
    \includegraphics[width=0.95\linewidth]{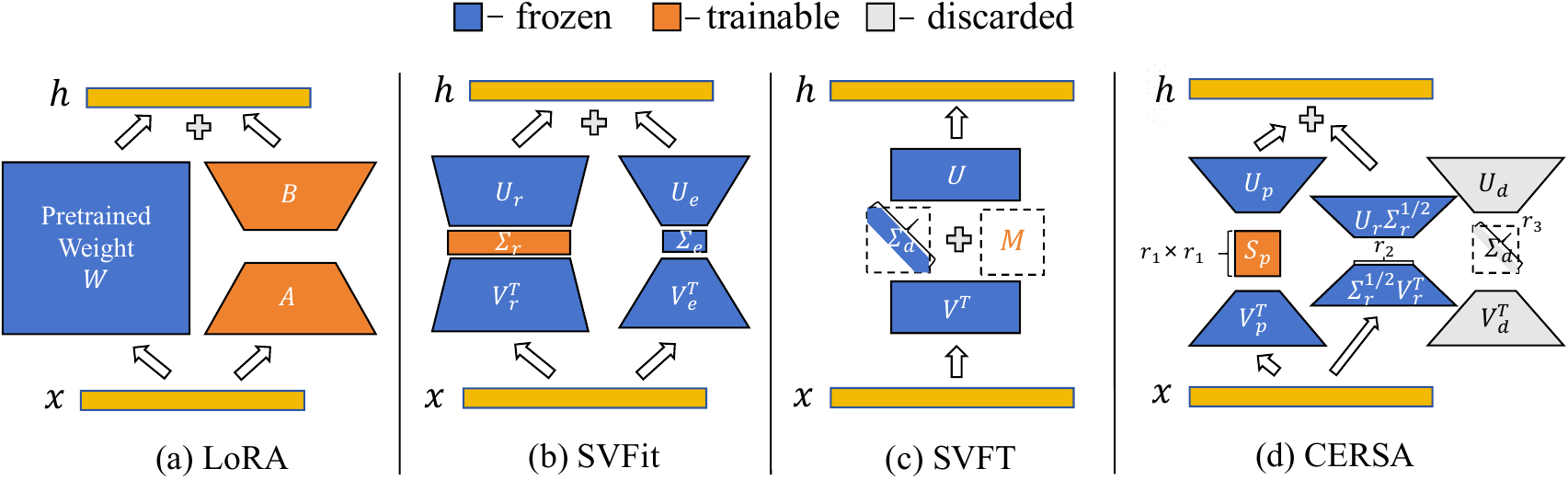}
    \vspace{-.12in}
    \caption{Comparison among LoRA~\citep{hu2022lora}, SVFit~\citep{sun2024svfit},
    SVFT~\citep{lingam2024svft} and CERSA. (a) LoRA~\citep{hu2022lora} uses two low-rank
    matrices to approximate weight updates
    % $\Delta \bm{W}$
    during fine-tuning. (b) SVFit~\citep{sun2024svfit} initializes low-rank matrices
    through SVD of $\bm{W}$ and trains only the most significant singular values
    as a vector. (c) SVFT~\citep{lingam2024svft} freezes singular vectors while sparsely
    fine-tuning singular values. (d) CERSA discards redundant SVD components and
    only trains a core matrix initialized with the most significant singular values. }
    \label{fig:architecture}
    \vspace{-.20in}
\end{figure*}

For a descending sequence of $N$ singular values: $\sigma_{1}^{2}\geq \sigma_{2}^{2}
\geq \dots \geq \sigma_{N}^{2}$, we define two hyperparameters, $\alpha$ and $\beta$,
to determine the preserved and trainable subspaces. The threshold $\alpha$ defines
the retained subspace, while $\beta$ specifies the trainable portion within that
subspace. $k_{\alpha}$ and $k_{\beta}$ are the smallest indices such that the cumulative
sum reaches the proportions $\alpha$ and $\beta$, respectively:
\vspace{-2mm}
\begin{equation}
    \label{eq:mid_val_k}k_{\alpha}= \min \left\{ k \,\middle|\, \frac{\sum_{i=1}^{k}\sigma_{i}^{2}}{\sum_{i=1}^{N}\sigma_{i}^{2}}
    \geq \alpha \right\}, \qquad k_{\beta}= \min \left\{ k \,\middle|\, \frac{\sum_{i=1}^{k}\sigma_{i}^{2}}{\sum_{i=1}^{N}\sigma_{i}^{2}}
    \geq \beta \right\}.
\end{equation}
These two thresholds, $\alpha$ and $\beta$, divide the matrix into three distinct
regions, as illustrated in \cref{fig:architecture}(d). \circled{1} \textbf{Discarded
Component:} We perform SVD on the pre-trained weights $\bm{W}$, \ie,
$\bm{W=USV^T}$. By setting the cumulative energy retention rate $\alpha$, the
least important components are discarded by truncating the least
$r_{3}= N-k_{\alpha}$ significant singular vectors in $\bm{U}$ and $\bm{V}$.
%the rank \(r_3\) is determined based on the energy retention rate of the matrix.
%The top-\(r_3\) components correspond to the retained cumulative energy portion, while the remaining components represent the discarded part.
Unlike PiSSA~\citep{meng2024pissa} and SVFit~\citep{sun2024svfit}, which retain
the residual part by freezing it, CERSA eliminates redundant high-rank components
that contribute only 5\%-10\% to the cumulative energy but occupying 50\%-90\% of
the embedding dimensions, as determined by the SVD cumulative energy truncation
index (\cref{fig:layerwise_svd}). Most of the singular values in this discarded portion
are near zero, indicating feature dimensions that are insignificant in the pre-trained
weights. Despite that this is still a lossy compression, these high-rank
components may not be well-aligned or optimally parameterized for downstream
fine-tuning tasks.
%Retaining them would unnecessarily complicate the parameter space, making optimization more challenging.
\circled{2} \textbf{Frozen Component}: Next, we introduce another hyperparameter,
$\beta$, as a threshold to compute the rank $r_{2}= k_{\alpha}- k_{\beta}$,
which determines which of the remaining principal components will be frozen. The
value of $\beta$ reflects the trade-off between preserving more pre-trained knowledge
and utilizing additional dimensions to learn the feature distribution of downstream
tasks for stronger fitting capability. In the image classification and text
sequence classification tasks, we set $\beta = \alpha$, allowing all principal
dimensions to participate in fine-tuning for maximum fitting performance. However,
if fine-tuning prioritizes retaining pre-trained knowledge, a smaller $\beta$ can
be chosen to freeze more dimensions. \circled{3} \textbf{Trainable Component}: The
remaining portion, with rank $r_{1}= k_{\alpha}$, is designated as trainable.
Unlike SVFit~\citep{sun2024svfit}, which focuses solely on learning the
distribution of singular values, we initialize the diagonal of the $\bm{S}_{p}$ matrix
with the top-$r_{1}$ singular values, while setting the remaining elements of
the $r_{1}\times r_{1}$ matrix to zero and making them trainable. This approach retains
critical singular values while allowing for complex linear combinations between the
left and right singular vectors, enhancing the model's expressive power.

By decomposing the matrix into three components, we enable fine-tuning within a much
smaller matrix, significantly reducing the number of trainable parameters. This approach
substantially lowers memory consumption and computational cost while preserving model
performance.

\subsection{Theoretical Analysis}
CERSA fine-tunes pre-trained weights with fewer parameters by decomposing them
via SVD and freezing the $\bm{U}$ and $\bm{V}$ matrices for parameter efficiency.
While methods like SVFit~\citep{sun2024svfit} and SVFT~\citep{lingam2024svft} follow
a similar approach, freezing $\bm{U}$ and $\bm{V}$ limits the model’s
expressiveness, making them suboptimal compared to full-parameter fine-tuning.
CERSA overcomes this by introducing a trainable matrix $\bm{S}_{p}$ initialized
with singular values on the diagonal, reducing these constraints while maintaining
memory efficiency. In the following, we establish the theoretical foundation of CERSA
to show that its performance closely matches the full-parameter fine-tuning.

For any weight matrix $\bm{W}\in \mathbb{R}^{m \times n}$ in a pre-trained model,
we define its full-parameter fine-tuned counterpart in downstream tasks as
$\bm{W}^{\prime}$. In the previous section, we removed the bottom 5\%–10\% of cumulative
energy from $\bm{W}$ by performing truncated SVD, as these correspond to
insignificant features and noise in the neural network. This allows us to approximate
the weight matrix as $\bm{W}\approx \bm{U}_{p}\bm{\Sigma}_{p}\bm{V}_{p}^{T}$.
Similarly, the fine-tuned weight matrix can be approximated as $\bm{W}^{\prime}\approx
\bm{U}_{p}^{\prime}\bm{\Sigma}_{p}^{\prime}\bm{V}_{p}^{\prime T}$.
% In contrast to LoRA-based methods, which use an entirely new distribution to represent the parameter update, CERSA uses the subspace of the pre-trained weights \( \bm{W}\) to represent \(  \bm{W}^\prime\) directly.

%To ensure that \( \bm{W} \) is adapted to downstream tasks within the principal subspace, one would need to train \( \bm{U}_p, \bm{\Sigma}_p, \bm{V}_p \) to obtain \( \bm{U}_p^\prime, \bm{\Sigma}_p^\prime, \bm{V}_p^{\prime} \).

\begin{theorem}
    \label{the:subspace} Given a matrix $\bm{M}$, applying the SVD, we have $\bm{M}
    = \bm{U}\bm{\Sigma}\bm{V}$. If there exists a pair of orthonormal bases
    $\bm{Q}= \{\bm{e}_{1}, \bm{e}_{2}, \dots, \bm{e}_{k}\}$ and $\bm{Q'}
    = \{\bm{e}'_{1}, \bm{e}'_{2}, \dots, \bm{e}'_{k}\}$ such that
    $\text{Span}(\bm{U}) = \text{Span}(\bm{Q})$, $\text{Span}(\bm
    {V}) = \text{Span}(\bm{Q'})$, there exists a matrix
    $\bm{S}\in \mathbb{R}^{k \times k}$ such that
    $\bm{M}= \bm{Q}\bm{S}\bm{Q}'^{T}$.
\end{theorem}

Unlike SVFit~\citep{sun2024svfit} and SVFT~\citep{lingam2024svft}, which assume
that $\bm{U}$ and $\bm{V}$ remain unchanged during fine-tuning, in practice,
$\bm{U}$ and $\bm{V}$ are likely to be updated to adapt to downstream tasks. Therefore,
we propose an alternative hypothesis: rather than $\bm{U}$ and $\bm{V}$ being strictly
invariant, the truly preserved components are the principal subspaces of
$\bm{U}_{p}^{\prime}$ and $\bm{V}_{p}^{\prime}$. This implies that the span of these
sets of singular vectors remains unchanged, \textit{i.e.},
$\text{Span}(\bm{U}_{p}^{\prime}) = \text{Span}(\bm{U}_{p}) \ \text{and}\ \text{Span}
(\bm{V}_{p}^{\prime}) = \text{Span}(\bm{V}_{p})$.
% Instead of updating \( \bm{U}_p \) and \( \bm{V}_p \), we freeze them and only train the intermediate matrix \( \bm{S}_p \in \mathbb{R}^{k \times k} \), which is equivalent to updating all three components \( \bm{U}_p\), \( \bm{\Sigma}_p \), and \( \bm{V}_p^T \).
According to \cref{the:subspace}, since the principal subspaces remain the same
before and after fine-tuning, there exists a transformation matrix $\bm{S}_{p}\in
\mathbb{R}^{k \times k}$ such that $\bm{W}^{\prime}\approx \bm{U}^{\prime}_{p}\bm
{\Sigma}^{\prime}_{p}\bm{V}^{\prime T}_{p}= \bm{U}_{p}\bm{S}_{p}\bm{V}_{p}^{T}$.
This suggests that rather than explicitly updating $\bm{U}_{p}$ and $\bm{V}_{p}$,
we can freeze them and only update the intermediate matrix $\bm{S}_{p}$.
%Therefore, training a tiny parameterized matrix \(\bm{S}_p \)
This is mathematically equivalent to updating all three components: $\bm{U}_{p}$,
$\bm{\Sigma}_{p}$, and $\bm{V}_{p}^{T}$, effectively removing the expressiveness
constraints imposed by freezing $\bm{U}$ and $\bm{V}$. The proof of this theorem
is provided in the Appendix (Sec. F.2).

To further illustrate the effect of a fully trainable matrix $\bm{S}_{p}\in \mathbb{R}
^{k \times k}$, as shown in \cref{fig:cersa-explain}, we apply an additional SVD
decomposition to the updated $\bm{S}_{p}$ after fine-tuning. The resulting
$\bm{V}_{\bm{s}_p}^{T}$, an $r_{1}\times r_{1}$ rotation matrix, rotates
$\bm{V}_{p}^{T}$ to adjust the spatial distribution of input features while
preserving the integrity of the subspace. Similarly, $\bm{U}_{\bm{s}_p}$ adjusts
the rotation of $\bm{U}_{p}$, concentrating key features within the intermediate
space along output directions essential for downstream tasks.

Additionally, our experiments (detailed in Appendix Sec. F.1) confirm that
in most full-parameter fine-tuning downstream tasks, the principal subspaces of $\bm
{W}$ and $\bm{W}^{\prime}$ exhibit a Grassmann subspace similarity~\citep{hu2022lora}
of 99\%$\sim$99.99\%. This provides strong empirical evidence supporting the assumption
that the principal subspaces of $\bm{W}$ remain nearly unchanged after fine-tuning.

% However, while the principal subspace of the pre-trained weights \(\bm{W}\) barely changes, restricting updates to only the singular values \(\bm{\Sigma_p}\) (as in SVFit~\citep{sun2024svfit}) imposes a significant limitation. Specifically, this approach merely scales the contributions of the pre-trained singular vectors without altering the fundamental input-output feature mappings. Consequently, such a restriction may hinder the model's ability to fully adapt to the input feature distributions and output representations required by downstream tasks. To address this, a trainable matrix \(\bm{S_p}\) enables more effective adaptation while maintaining the original subspace structure.
% (see \cref{fig:cersa} in the appendix).
%, which is mapped by \(U_w\) and \(V_w\).
%However, if the previously zero elements in the diagonal matrix are made trainable, it effectively allows the simultaneous optimization of \(U_w\) and \(V_w\). This enables the model to adapt both input feature distributions and output feature representations, significantly enhancing its flexibility and its ability to address the demands of downstream tasks.

Although this approach increases the number of trainable parameters from $k$ to $k^{2}$ (compared to SVFit~\citep{sun2024svfit}), the space saved through compression in the pre-trained model and the performance gains achieved make this increase a justifiable cost. The larger parameterization also enables stronger adaptability, as evidenced by a faster loss decrease during fine-tuning (see \cref{fig:loss-curve}).
At the same time, the benefits of fixing $\bm{U}_{p}$ and $\bm{V}_{p}$ are preserved, since the rotation matrix preserves the input and output subspace, only changing its basis representation. This ensures that features learned during pre-training remain intact, with only their distribution adjusted.

\begin{figure*}[!t]
    \centering
    %—— 第1列：图 1 ——%
    \begin{minipage}[c]{0.55\textwidth}
        \centering
        \includegraphics[width=\linewidth]{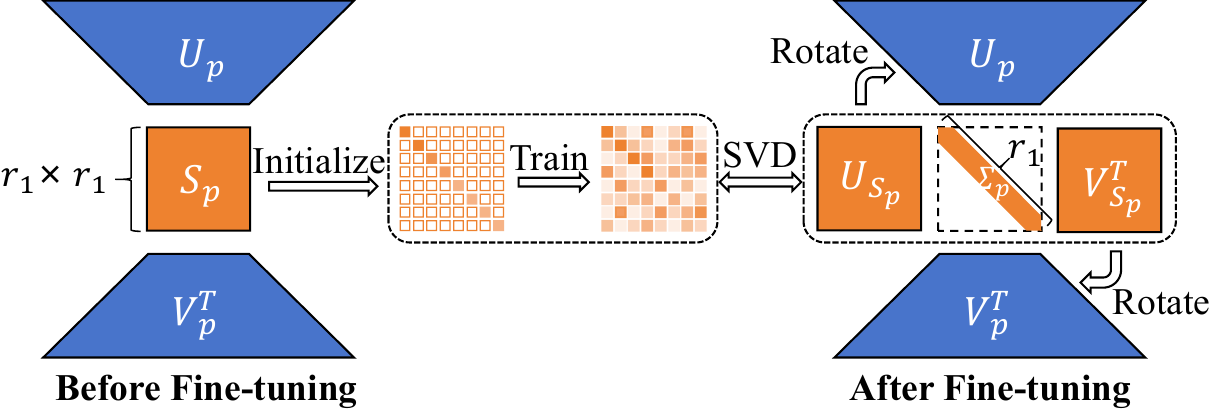}
        \vspace{-.12in}
        \captionof{figure}{Training process of CERSA. The trainable core matrix \(\bm{S}_{p}\) can be decomposed into \(\bm{U}_{\bm{s}_p}\) and \(\bm{V}_{\bm{s}_{p}}\), enabling fine-tuning by rotating the input and output bases without altering the subspace itself.} \label{fig:cersa-explain}
    \end{minipage}\hfill
    %—— 第2列：先表1再表2 ——%
    \begin{minipage}[c]{0.425\textwidth}
        \centering
        \includegraphics[width=0.9\linewidth]{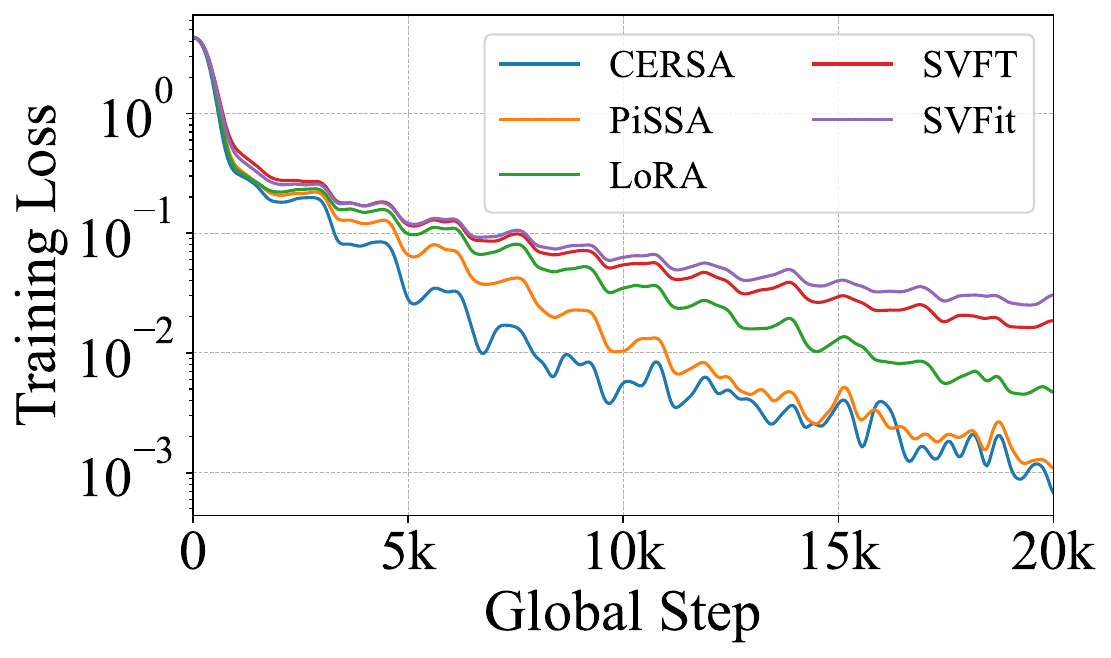}
        \vspace{-0.18in}
        \captionof{figure}{Loss curve of fine-tuning ViT-Large~\citep{dosovitskiy2021image} on CIFAR-100~\citep{krizhevsky2009learning} using various methods.}
        \label{fig:loss-curve}
    \end{minipage}
    \vspace{-0.1in}
\end{figure*}

\subsection{Memory Efficiency}
The primary goal of applying SVD to the pre-trained weight matrix is to reduce memory
consumption during fine-tuning. Consequently, a compression rank threshold $b$ exists, below which memory savings are achieved only if $r < b$. Given a pre-trained weight
matrix $\bm{W}\in \mathbb{R}^{m \times n}$, its truncated SVD decomposition
produces: $\bm{U}\in \mathbb{R}^{m \times r}$,
$\bm{\Sigma}\in \mathbb{R}^{r \times r}$, and $\bm{V}\in \mathbb{R}^{r \times n}$,
such that $\bm{W = U \Sigma V}$.

\begin{table}[!t]
    % \vspace{-.1in}
    \centering
    \footnotesize
    \setlength{\tabcolsep}{7pt}
    \renewcommand{\arraystretch}{1.0}
    \resizebox{.8\linewidth}{!}{%
    \begin{tabular}{@{}lccccc@{}}
        \toprule         & FT    & CERSA (Ours)       & SVFit          & SVFT       & LoRA             \\
        \midrule Weights & $mn$  & $mr + nr + r^{2}$  & $2mn + m$      & $2mn + e$  & $mn + mr + nr$   \\
        Gradients        & $mn$  & $r^{2}$            & $r$            & $e$        & $mr + nr$        \\
        Opt.\ states     & $2mn$ & $2r^{2}$           & $2r$           & $2e$       & $2mr + 2nr$      \\
        Total            & $4mn$ & $mr + nr + 4r^{2}$ & $2mn + m + 3r$ & $2mn + 4e$ & $mn + 4mr + 4nr$ \\
        \bottomrule
    \end{tabular}}
    \vspace{-3pt}
    \caption{Memory requirements. In SVFT, $e$ is the number of sparsified
    trainable parameters from the diagonal after SVD, where $e \ll mn$ but $e > m$.}
    \label{tab:memory_analysis}
    \vspace{-.2in}
\end{table}

During training, memory usage consists of the frozen matrices $\bm{U}$ and $\bm{V}$,
along with the trainable matrix $\bm{S}$, requiring $\mathcal{O}(mr + nr + r^{2})$
storage. Additionally, since $\bm{S}$ is trainable, its gradient and optimizer
states contribute $\mathcal{O}(r^{2})$ and $\mathcal{O}(2r^{2})$, respectively,
leading to a total memory requirement of $\mathcal{O}(mr + nr + 4r^{2})$. \cref{tab:memory_analysis}
compares memory costs across different methods.

In contrast, the original weight matrix $\bm{W}$ requires $\mathcal{O}(mn)$
memory, excluding the additional $\mathcal{O}(3mn)$ for gradient and optimizer states,
as we focus on reducing memory relative to the pre-trained parameters. The compression
rate $c$ depends on the model's input-output dimensions, with a lower value
indicating better memory efficiency: $c = \frac{mr + nr + 4r^{2}}{mn}$. Based on
this, we compute the compression curves for three variants of the ViT model:
Base, Large, and Huge, across cumulative energy retention rates ranging from 0.8
to 0.95, as shown in \cref{fig:vit-compression-rate}. The calculations are performed
for two target module configurations: one with all three Q, K, and V matrices
and another with only Q and V matrices. The dashed horizontal line in \cref{fig:vit-compression-rate}
represents the compression rate achieved by the LoRA method when fine-tuning only
Q and V matrices with a rank of 32. The results indicate that for smaller models,
such as ViT-Base, the compression rate is less favorable, possibly due to the
limited embedding dimension. However, for larger models like ViT-Large and ViT-Huge~\citep{dosovitskiy2021image},
there is considerably more room for compression, enabling greater memory efficiency.

    \section{Experiments}
We conduct extensive evaluations on both image classification and natural language understanding (NLU) tasks. Further experimental results, including subject-driven text-to-image generation and out-of-distribution evaluations, are provided in the appendix(see Sec. A.1 and Sec. E).

\subsection{Experimental Setup}
%
% \noindent\textbf{Implementation.}\quad
%
\para{Baseline selection.} For baseline comparisons, we include full-parameter fine-tuning
(FT), popular PEFT methods such as LoRA~\citep{hu2022lora} and PiSSA~\citep{meng2024pissa},
and weight-decomposition-based approaches like SVFit~\citep{sun2024svfit} and
SVFT~\citep{lingam2024svft}.

\para{Model selection.} For image classification, we evaluate ViT-Base and ViT-Large~\citep{dosovitskiy2021image},
pre-trained on ImageNet-21K \citep{deng2009imagenet}. For the NLU experiment, we fine-tune DeBERTaV3-Base ~\citep{he2023debertav3}
to assess the fundamental capabilities of our method.

% abation study

\begin{figure*}[!t]
    \centering
    %—— 第一列：图 1 ——%
    \begin{minipage}[c]{0.48\textwidth}
        \centering
        \includegraphics[width=0.9\linewidth]{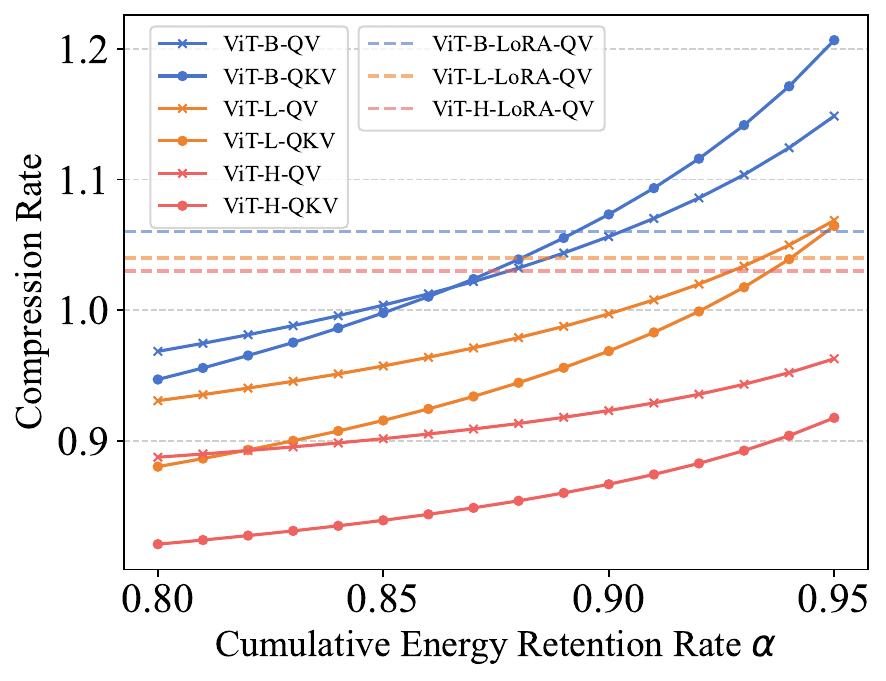}
        \vspace{-.18in}
        \captionof{figure}{Comparison on ViT compression rates across various cumulative energy retention rates.}
        \label{fig:vit-compression-rate}
    \end{minipage}\hfill
    %—— 第三列：先表1再表2 ——%
    \begin{minipage}[c]{0.48\textwidth}
        \centering
        \resizebox{\linewidth}{!}{%
            \begin{tabular}{lccccc}
            \toprule Method           & CIFAR-100     & RESISC45      & DTD           & Average       & Total Memory \\
            \midrule CERSA(Q, V)      & 94.0          & 95.8          & 82.1          & 90.6          & 1194.5 MB    \\
            CERSA(Q, K, V)            & \textbf{94.4} & \textbf{96.1} & 82.5          & \textbf{91.0} & 1232.9 MB    \\
            CERSA(Q, K, V, P)         & \textbf{94.5} & 96.0          & \textbf{82.6} & \textbf{91.0} & 1279.5 MB    \\
            CERSA(Q, K, V, P, UP, DN) & 93.8          & 94.9          & 81.6          & 90.1          & 1433.1 MB    \\
            \bottomrule
        \end{tabular}}
        \vspace{-0.1in}
        \captionof{table}{Results of CERSA with various matrix type combinations for the ViT-Large model~\citep{dosovitskiy2021image}. For the definitions of (Q, K, V, P, UP, DN), please refer to~\cref{fig:layerwise_svd}.}
        \label{tab:matrix_types}
        
        \vspace{0.1in}

        \resizebox{\linewidth}{!}{%
            \begin{tabular}{lccccc} \toprule
            Method & CIFAR-100 & RESISC45 & DTD  & Average    & Total Memory \\
            \midrule
            Top-\(r_1\) & \textbf{93.4} & \textbf{95.6} & \textbf{81.3} & \textbf{90.1} & 1112.7 MB \\
            Bottom-\(r_2\)       & 92.6     & 94.7     & 79.9  & 89.1   & 1112.7 MB     \\
            \bottomrule
            \end{tabular}}
            \vspace{-.1in}
            \captionof{table}{Results of fine-tuning on Top-\(r_1\) and Bottom-\(r_2\) rank for the ViT-Large model~\citep{dosovitskiy2021image}.}
            \label{tab:top-bottom}
        \end{minipage}
        \vspace{-0.12in}
\end{figure*}

\begin{figure*}[!t]
    \vspace{-.05in}
    \centering
    \begin{minipage}[c]{0.48\textwidth}
        \centering
        \resizebox{0.98\linewidth}{!}{
        \begin{tabular}{lccccc}
            \toprule Method                                 & CIFAR-100     & RESISC45      & DTD           & Average       & Total Memory \\
            \midrule $\text{CERSA}_{\alpha=1,\beta=1}$      & 93.8          & \textbf{96.3} & 81.8          & 90.6          & 2519.6 MB    \\
            $\text{CERSA}_{\alpha=0.95,\beta=0.95}$         & \textbf{94.3} & 96.1          & \textbf{82.5} & \textbf{91.0} & 1232.9 MB    \\
            $\text{CERSA}_{\alpha=0.9,\beta=0.9}$           & 93.9          & 96.1          & 82.1          & 90.7          & 1122.5 MB    \\
            $\text{CERSA}_{\alpha=0.8,\beta=0.8}$           & 93.5          & 95.9          & 81.8          & 90.4          & 1020.6 MB    \\
            $\text{CERSA}_{\alpha=0.5,\beta=0.5}$           & 90.0          & 95.1          & 79.5          & 88.2          & 914.9 MB     \\
            \midrule $\text{CERSA}_{\alpha=0.95,\beta=0.9}$ & 94.0          & 96.0          & \textbf{82.5} & 90.8          & 1169.4 MB    \\
            $\text{CERSA}_{\alpha=0.95,\beta=0.8}$          & 93.8          & 96.1          & 82.2          & 90.7          & 1118.2 MB    \\
            $\text{CERSA}_{\alpha=0.95,\beta=0.5}$          & 92.9          & 95.2          & 80.3          & 89.5          & 1079.3 MB    \\
            \bottomrule
        \end{tabular}
        }
        \vspace{-.1in}
        \captionof{table}{Results of various cumulative energy retention rates for the ViT-Large model \citep{dosovitskiy2021image} across the CIFAR-100 \citep{krizhevsky2009learning}, RESISC45 \citep{cheng2017remote}, and DTD \citep{cimpoi2014describing} datasets.}
        \label{tab:cer-rate}
        \vspace{-.12in}
    \end{minipage}\hfill
    \begin{minipage}[c]{0.48\textwidth}
        \centering
        \resizebox{\linewidth}{!}{
        \begin{tabular}{lccccc}
            \toprule Method                          & CIFAR-100     & RESISC45      & DTD           & Average       & Total Memory \\
            \midrule Layer-wise ($\alpha=\beta=0.9$) & \textbf{93.9} & \textbf{96.1} & \textbf{82.1} & \textbf{90.7} & 1122.5 MB    \\
            Uniform ($r=287$)                        & 93.7          & 95.6          & 81.4          & 90.2          & 1122.5 MB    \\
            \bottomrule
        \end{tabular}}
        \vspace{-.05in}
        \captionof{table}{Results of layer-wise and uniform rank for the ViT-Large model \citep{dosovitskiy2021image}. }
        \label{tab:uniform-layerwise} 
        
        \resizebox{\linewidth}{!}{%
        \begin{tabular}{lccccc}
            \toprule Method          & CIFAR-100     & RESISC45      & DTD           & Average       & Total Memory \\
            \midrule CERSA w/ Matrix & \textbf{93.9} & \textbf{96.1} & \textbf{82.1} & \textbf{90.7} & 1122.5 MB    \\
            CERSA w/ Array           & 93.5          & 95.2          & 81.5          & 90.0          & 1045.5 MB    \\
            \bottomrule
        \end{tabular}}
        \vspace{-.05in}
        \captionof{table}{Results of CERSA with a trainable matrix or array for the ViT-Large model \citep{dosovitskiy2021image}.}
        \label{tab:matrix_array}
        \vspace{-.12in}
    \end{minipage}
    \vspace{-.12in}
\end{figure*}

\para{Datasets.} For image classification, we assess our method on eight diverse
datasets: CIFAR-100~\citep{krizhevsky2009learning}, EuroSAT \citep{helber2019eurosat},
RESISC45~\citep{cheng2017remote}, StanfordCars~\citep{krause20133d}, FGVC
Aircraft~\citep{maji2013fine}, DTD~\citep{cimpoi2014describing}, CIFAR-10~\citep{krizhevsky2009learning},
and OxfordPets~\citep{parkhi2012cats}. These datasets span a variety of
classification tasks, including general object classification, fine-grained
classification, remote sensing image classification, and texture classification.

For NLU, we evaluate our method on eight datasets from the GLUE benchmark \citep{wang2019glue}:
MNLI, MRPC, RTE, CoLA, SST-2, QNLI, QQP, and STS-B. These datasets cover a broad
spectrum of NLU tasks, including textual entailment, paraphrase detection,
sentiment analysis, question-answer matching, and semantic textual similarity.

\para{Metrics.} For image classification, we report accuracy across all datasets.
For NLU, we report overall matched and mismatched accuracy on MNLI, Matthew's correlation
on CoLA, Pearson correlation on STS-B, and accuracy on the remaining datasets.
Higher values indicate better performance.

\subsection{Ablation Study}
\noindent

\para{Impact of the matrix type.} We study the trade-off between performance and memory when fine-tuning different matrix types. As shown in \cref{tab:matrix_types}, adapting Q, K, and V achieves the best balance of accuracy and efficiency. In contrast, adding P, UP, or DN increases memory cost and even reduces performance. This is mainly because: (i) these matrices require higher ranks to preserve cumulative energy (\cref{fig:layerwise_svd}), making them less memory-efficient; and (ii) modifying them disrupts pre-trained feature representations, leading to overfitting or weaker generalization. Thus, restricting fine-tuning to Q, K, and V provides the optimal trade-off.

\para{Impact of top-\(r_1\) versus bottom-\(r_2\) ranks.} To assess the effect of fine-tuning major versus residual components, we compare $\bm{S}_p$ trained on the top-$r_1$ and bottom-$r_2$ components (with $r_1 = r_2$, $\alpha=0.95$). As shown in Tab.~\ref{tab:top-bottom}, adapting the top components consistently yields higher accuracy, validating our design. Moreover, the OOD results in the appendix(see  Sec. E) show that CERSA surpasses both LoRA and FT, confirming that it preserves rather than distorts the original knowledge.

\para{Impact of the cumulative energy retention rate.} 
We investigate how different retention rates affect fine-tuning by training ViT-Large~\citep{dosovitskiy2021image} under various $\alpha$ and $\beta$ configurations (\cref{tab:cer-rate}). The first five settings use $\alpha=\beta$ decreasing from 1 to 0.5, while the last three fix $\alpha=0.95$ and vary $\beta$ to adjust the trainable subspace size. Among them, $\text{CERSA}_{\alpha=0.95,\beta=0.95}$ achieves the best overall accuracy. Performance remains stable even at $\alpha=\beta\approx0.8$ (90\% of pre-trained memory), but drops sharply when reduced to 0.5 or lower.

\para{Impact of layer-wise versus uniform.} We compare layer-wise CERSA, which selects singular values by each layer's cumulative energy retention rate, with uniform CERSA, which fixes the rank at 287 for all layers. As shown in \cref{tab:uniform-layerwise}, despite identical memory consumption, layer-wise CERSA consistently outperforms the uniform variant, demonstrating the effectiveness of exploiting layer-specific retention rates.

\para{Impact of tuning a matrix versus array.} We analyze the impact of defining the $\bm{S}$ matrix as either a matrix or an array (as in SVFit~\citep{sun2024svfit}) on fine-tuning performance
% We compared fine-tuning with matrices versus arrays as used in SVFit~\citep{sun2024svfit}
under the same CERSA configuration $\alpha=0.9,\beta=0.9$ for Q, K, and V. As
shown in \cref{tab:matrix_array}, although arrays significantly reduce memory usage,
they result in a substantial drop in performance, highlighting that matrix-based
fine-tuning has much better expressiveness.

% --------------------------------------------------------
% \begin{figure*}[!t]
%     \centering
%     \includegraphics[width=0.96\linewidth]{Imgs/dreambooth.pdf}
%     \vspace{-.10in}
%     \caption{Visual comparison of images generated by the fine-tuned subject-driven
%     diffusion model using DreamBooth~\citep{ruiz2023dreambooth} (full-parameter
%     fine-tuning), LoRA~\citep{hu2022lora}, and the proposed CERSA.}
%     \label{fig:dreambooth}
%     \vspace{-.05in}
% \end{figure*}
% --------------------------------------------------------

\begin{table*}
    [!t]
    \centering
    \small
    \resizebox{\linewidth}{!}{%
    \begin{tabular}{l|c|ccccccccc}
        \toprule Method & Memory & CIFAR-100     & EuroSAT       & RESISC45         & StanfordCars     & FGVC-Aircraft & DTD           & CIFAR-10      & OxfordPets    & Average       \\
        \midrule FT  & 4629.8 MB      & 93.6          & 99.0          & \underline{96.4} & \underline{88.9} & 68.3          & 81.8          & 99.2          & 94.4          & 90.2       \\
        LoRA   & 1229.9 MB         & \textbf{94.9} & 99.0          & 94.7             & 80.3             & 54.5          & 81.5          & 99.1          & 94.8        & 87.4 
         \\
        PiSSA  & 1229.9 MB         & 93.6          & 98.6          & 95.7             & 86.7             & 62.6          & 81.8          & 98.8          & \textbf{95.7} & 89.2     \\
        SVFit  & 1351.5 MB           & 93.9          & 98.7          & 95.2             & 83.3             & 57.8          & 81.5          & \textbf{99.3} & 93.4          & 88.7       \\
        SVFT  & 1355.8 MB           & 93.7          & 98.8          & 95.4             & 84.9             & 63.7          & 82.3          & \textbf{99.3} & 93.3          & 88.9        \\
        CERSA  & 1232.3 MB         & 94.3          & \textbf{99.1} & \textbf{96.1}    & \textbf{87.6}    & \textbf{71.1} & \textbf{82.5} & \textbf{99.3} & 94.9          & \textbf{90.6}  \\
        \bottomrule
    \end{tabular}
    } % end resizebox
    \vspace{-.1in}
    \caption{Comparison of various fine-tuning methods on eight image
    classification datasets using ViT-Large~\citep{dosovitskiy2021image}.
    Methods include LoRA~\citep{hu2022lora}, PiSSA~\citep{meng2024pissa}, SVFit~\citep{sun2024svfit},
    and SVFT~\citep{lingam2024svft}. Bold scores indicate the highest accuracy
    among PEFT methods, while underlined scores indicate that full-parameter fine-tuning
    (FT) achieves the best performance.}
    \label{tab:vit_large_results}
    \vspace{-.15in}
\end{table*}

\begin{table*}
    [!t]
    \centering
    \small
    \resizebox{0.95\linewidth}{!}{%
    \begin{tabular}{l|c|ccccccccc}
        \toprule Method & Memory & MNLI          & MRPC          & STS-B         & RTE           & SST-2         & QNLI          & QQP              & CoLA          & Average       \\
        \midrule FT  & 2814.0 MB    & 89.9          & 89.5          & 91.6          & 83.8          & 95.6          & 94.0          & \underline{92.4} & 69.2          & 88.3          \\
        LoRA    & 730.5 MB        & \textbf{90.7} & 90.0          & 91.6          & 85.2          & 95.0          & 93.9          & 92.0             & 69.8          & 88.5          \\
        PiSSA   & 730.5 MB       & 90.4          & 91.7          & \textbf{91.9} & 87.0          & 95.9          & 94.3          & 92.3             & \textbf{72.6} & \textbf{89.5} \\
        SVFit   & 1096.3 MB       & 89.7          & 88.8          & 91.8          & \textbf{87.4} & 95.4          & 94.3          & 90.2             & 71.0          & 88.6          \\
        SVFT   & 1108.6 MB        & 90.0          & 89.0          & 91.8          & 87.2          & 95.4          & 94.3          & 91.5             & \textbf{72.6} & 89.0          \\
        CERSA  & 728.6 MB        & 90.3          & \textbf{92.0} & 91.7          & 87.0          & \textbf{96.0} & \textbf{94.4} & \textbf{92.4}    & 72.3          & \textbf{89.5} \\
        \bottomrule
    \end{tabular}}
    \vspace{-.1in}
    \caption{Comparison of different methods on the GLUE benchmark using the
    DeBERTaV3-Base model~\citep{he2023debertav3}. Methods include LoRA~\citep{hu2022lora},
    PiSSA~\citep{meng2024pissa}, SVFit~\citep{sun2024svfit}, and SVFT~\citep{lingam2024svft}.}
    \label{tab:deberta_results}
    \vspace{-.2in}
\end{table*}

\subsection{Comparison on Image Classification Tasks}
Experimental results for the ViT-Large model~\citep{dosovitskiy2021image} are
presented in \cref{tab:vit_large_results}. With ViT-Large~\citep{dosovitskiy2021image},
CERSA achieves an average accuracy of 90.6\%, outperforming full-parameter fine-tuning
(90.2\%) and significantly surpassing other PEFT methods like SVFT~\citep{lingam2024svft}
(88.9\%) and SVFit~\citep{sun2024svfit} (88.7\%). Notably, CERSA excels on fine-grained
classification tasks, particularly for datasets like StanfordCars~\citep{krause20133d}
and FGVC Aircraft~\citep{maji2013fine}, highlighting its capability to capture
intricate details. Besides, CERSA matches or exceeds full-parameter fine-tuning
on general datasets like CIFAR-100~\citep{krizhevsky2009learning}, EuroSAT~\citep{helber2019eurosat},
and RESISC45~\citep{cheng2017remote}, demonstrating its strong generalization
and adaptability across diverse tasks.
% Overall, CERSA surpasses existing PEFT methods and frequently matches or outperforms FT.
% Its strong performance on fine-grained classification highlights its capability to capture intricate details, while its robust results on general datasets confirm its adaptability across diverse tasks.

% Additionally, CERSA generates more consistent images under the same random seed. For instance, in ``A [V] cat with a blue house in the background'' both CERSA and full fine-tuning correctly depict the blue house, whereas LoRA generates only a vague blue patch.

\subsection{Comparison on NLU Tasks}
\cref{tab:deberta_results} compares fine-tuning strategies on eight GLUE datasets using DeBERTaV3-Base~\citep{he2023debertav3}. CERSA achieves the highest average score (89.5\%), outperforming both full-parameter fine-tuning and other PEFT methods, and sets new state-of-the-art results on multiple datasets, including MRPC, SST-2, QNLI, and QQP. For the remaining tasks, it also attains competitive performance compared with state-of-the-art methods PiSSA~\citep{meng2024pissa}, SVFit~\citep{sun2024svfit}, and SVFT~\citep{lingam2024svft}. This highlights the effectiveness of CERSA in leveraging pre-trained representations with minimal computational overhead, making it a strong choice for NLU tasks.

    \section{Conclusion}
We propose CERSA, a memory- and parameter-efficient fine-tuning method that performs layer-wise rank selection based on the cumulative energy retention of pre-trained weights, enabling adaptation within the principal subspace. We prove that CERSA achieves performance comparable to full fine-tuning, and extensive experiments on image classification and language understanding show it outperforms or matches state-of-the-art PEFT methods while reducing memory.  

\para{Limitation and Future Work.} Since CERSA constrains fine-tuning within the principal subspace of pre-trained weights, its performance may degrade when the downstream task significantly deviates from the knowledge captured during pre-training. In the future, we plan to extend CERSA's capabilities by dynamically adjusting its learned subspace during fine-tuning, thereby enhancing its adaptability and performance across a broader range of downstream tasks.
    
    \clearpage
    
    \bibliography{iclr2026_conference}
    \bibliographystyle{iclr2026_conference}

    \clearpage
    
    \appendix
    \section*{Appendix}   % 无编号的大标题
    \FloatBarrier
    \setcounter{section}{0}  % 保证从 A 开始
    \renewcommand{\thesection}{\Alph{section}}  % 用 A, B, C 编号
    % \section{Comparison between SVFit and CERSA}
% \begin{figure*}[!t]
%     \centering
%     \includegraphics[width=0.95\linewidth]{Imgs/CERSA-SVFit.pdf}
%     \vspace{-.10in}
%     \caption{Training process of CERSA \vs SVFit.}
%     \label{fig:cersa-compare}
%     \vspace{-.10in}
% \end{figure*}

% As shown in \cref{fig:cersa-compare}, CERSA retains the singular values obtained from the pre-trained decomposition and uses them to initialize the diagonal elements of the \( S_p \) matrix, while the remaining non-diagonal elements are initialized to zero, allowing them to be trained. After training, since any matrix can be written in the form of SVD decomposition, applying SVD to \( S_p \) results in a form that can be interpreted as the product of an orthogonal matrix \(V_{s_p}^T \)that rotates the input space spanned by \( V \), a new spectral energy distribution, and another orthogonal matrix \(U_{s_p} \) that rotates the output space \( U \). SVFit \citep{sun2024svfit} only adapts the singular values that redistribute the spectral energy in the fixed original direction of input and output space, limiting its expressiveness.

\section{Additional Experimental Results}
\subsection{Subject-driven Text-to-Image Generation}
\begin{figure*}
    \centering
    \includegraphics[width=\linewidth]{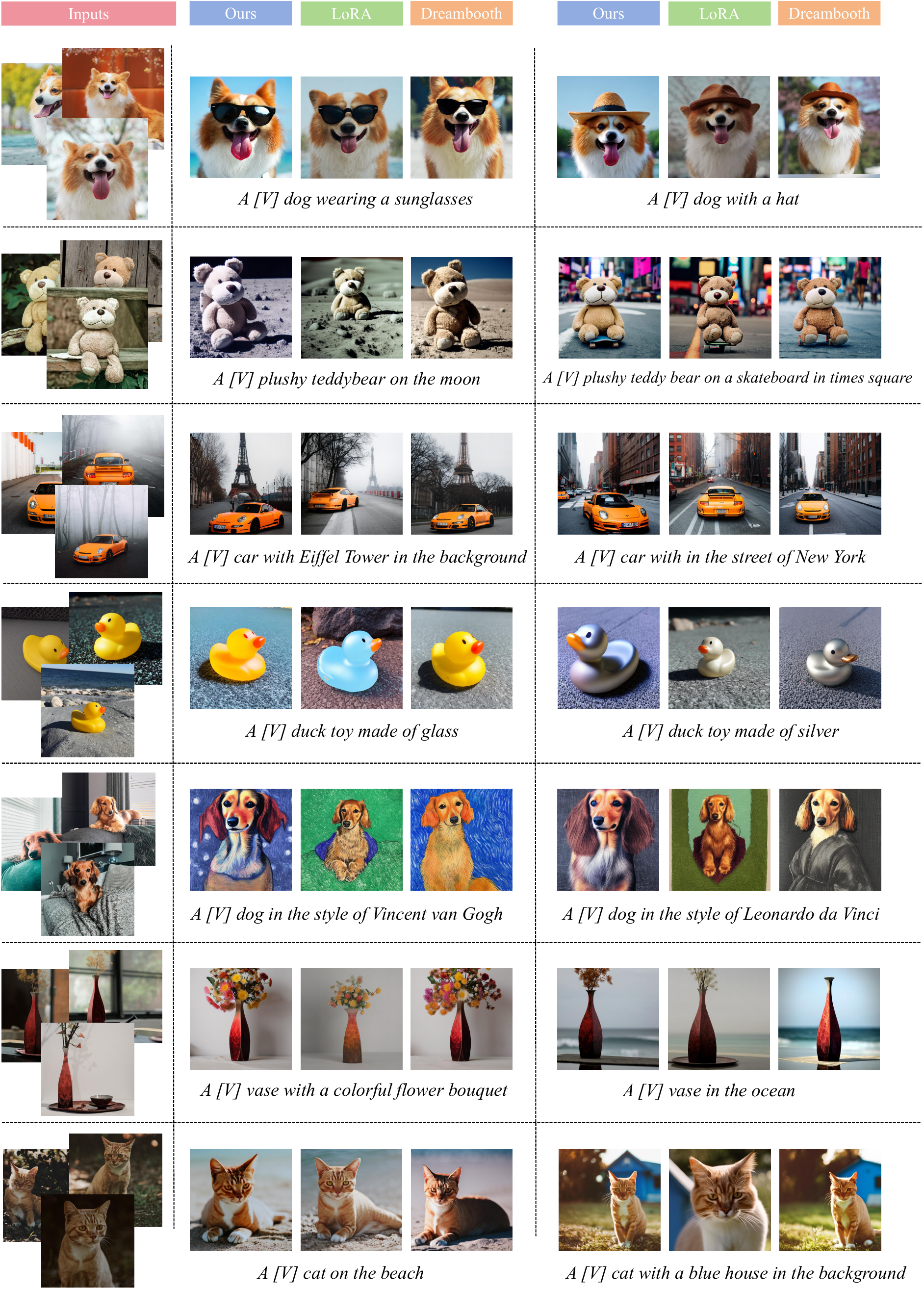}
    \caption{Results of visual comparison generated by the subject-driven
    fine-tuned diffusion model using the proposed CERSA, LoRA~\citep{hu2022lora},
    and DreamBooth~\citep{ruiz2023dreambooth}.}
    \label{fig:dreambooth-full}
    \vspace{-0.2in}
\end{figure*}

For subject-driven text-to-image generation, we fine-tune models using selected
samples from the DreamBooth~\citep{ruiz2023dreambooth} and CustomConcept101~\citep{Kumari2022MultiConceptCO}
datasets. Each subject sample contains 5 to 6 images captured from different angles
and contexts. We compare full-parameter fine-tuning (FT), LoRA~\citep{hu2022lora}, and our proposed CERSA on this task.

As shown in \cref{fig:dreambooth-full}, we evaluate subject-driven generation
across multiple domains, including scene composition, material modification, and
artistic style transfer:
\begin{itemize}
    \item \textbf{Scene composition.} When placing a sports car in front of the Eiffel Tower or on a New York street, CERSA captures both background details and subject fidelity more accurately than LoRA, producing results that more closely resemble FT.
    \item \textbf{Material modification.} Applying glass and silver textures to a duck toy highlights CERSA’s strength: it preserves the subject’s original shape and features while achieving consistent material transfer. In contrast, LoRA and FT often distort shapes or fail to maintain color/material consistency.
    \item \textbf{Style transfer.} When adapting a dog’s image into the styles of Vincent van Gogh and Leonardo da Vinci, all three methods demonstrate recognizable style transfer, but CERSA produces visuals that align more closely with FT while avoiding artifacts.
\end{itemize}

\paragraph{Quantitative comparison.}
We use CLIPScore~\citep{hessel2022clipscorereferencefreeevaluationmetric} to assess prompt-image alignment. CERSA achieves the highest
average CLIPScore(\textbf{32.75}), outperforming LoRA (31.88) and FT (32.35), indicating better generation quality.

Overall, these results demonstrate that CERSA achieves high-quality subject-driven
image generation, consistently surpassing LoRA and closely matching or even
exceeding the performance of full-parameter fine-tuning, while being significantly
more memory- and parameter-efficient.

\subsection{Image Classification}

\begin{table*}[t]
    \centering
    \small
    \resizebox{\linewidth}{!}{%
    \begin{tabular}{lccccccccc}
        \toprule                     % \\ \hline
        Method                      & CIFAR-100        & EuroSAT          & RESISC45         & StanfordCars  & FGVC-Aircraft & DTD           & CIFAR-10         & OxfordPets    & Average       \\
        \midrule FT                 & \underline{92.4} & \underline{99.1} & \underline{96.1} & 79.8          & 54.8          & 77.7          & \underline{98.9} & 93.1          & 86.5          \\
        LoRA     & 92.0             & 98.4             & 92.7             & 45.5          & 25.2          & 75.0          & \textbf{98.8}    & 93.1          & 77.6          \\
        PiSSA & 91.2             & 98.7             & 95.5             & 67.1          & 47.6          & 78.7          & 98.6             & \textbf{95.9} & 84.2          \\
        SVFit  & 91.6             & 98.6             & 93.0             & 67.2          & 47.9          & 80.5          & \textbf{98.8}    & 92.3          & 83.7          \\
        SVFT & 91.2             & 98.5             & 92.4             & 67.5          & 56.2          & 79.8          & 98.7             & 92.5          & 84.6          \\
        CERSA                       & \textbf{92.1}    & \textbf{98.9}    & \textbf{95.6}    & \textbf{83.9} & \textbf{68.2} & \textbf{81.2} & \textbf{98.8}    & 93.2          & \textbf{89.0} \\
        \bottomrule
    \end{tabular}
    }
    \vspace{-.1in}
    \caption{Comparison of various fine-tuning methods on eight image
    classification datasets using ViT-Base~\citep{dosovitskiy2021image}.
    Methods include LoRA~\citep{hu2022lora}, PiSSA~\citep{meng2024pissa}, SVFit~\citep{sun2024svfit}, and SVFT~\citep{lingam2024svft}. Bold scores indicate the highest accuracy among PEFT methods, while underlined scores indicate that full-parameter fine-tuning (FT) achieves the best performance.}
    \label{tab:vit_base_results}
    \vspace{-.1in}
\end{table*}

\begin{table*}[t]
    \centering
    % \small
    \resizebox{\linewidth}{!}{%
    \begin{tabular}{lccccccccc}
        \toprule                                          % \\\hline
                                                         & CIFAR-100     & EuroSAT       & RESISC45      & StanfordCars  & FGVC-Aircraft & DTD           & CIFAR-10      & OxfordPets    & Average \\
        \midrule $\text{CERSA}_{\alpha=1,\beta=1}$       & 91.3          & 97.6          & 85.5          & 72.8          & 64.6          & 78.9          & 98.6          & 92.3          & 85.2    \\
        $\text{CERSA}_{\alpha=0.95,\beta=0.95}$          & \textbf{92.1} & \textbf{98.9} & \textbf{95.6} & \textbf{83.9} & \textbf{68.2} & \textbf{81.2} & \textbf{98.8} & \textbf{93.2} & 89.0    \\
        $\text{CERSA}_{\alpha=0.9,\beta=0.9}$            & \textbf{92.1} & 98.6          & 95.3          & 83.5          & 68.2          & 80.3          & 98.6          & \textbf{93.2} & 88.7    \\
        $\text{CERSA}_{\alpha=0.8,\beta=0.8}$            & 91.1          & 98.1          & 94.9          & 80.2          & 67.4          & 78.1          & 98.5          & 92.6          & 87.6    \\
        $\text{CERSA}_{\alpha=0.5,\beta=0.5}$            & 83.3          & 95.4          & 91.7          & 62.9          & 50.7          & 67.4          & 96.3          & 87.7          & 79.4    \\
        \midrule $\text{CERSA}_{\alpha=0.95,\beta=0.95}$ & 92.1          & \textbf{98.9} & \textbf{95.6} & \textbf{83.9} & 68.2          & \textbf{81.2} & \textbf{98.8} & \textbf{93.2} & 89.0    \\
        $\text{CERSA}_{\alpha=0.95,\beta=0.9}$           & 92.1          & 98.6          & 95.3          & 83.7          & 69.8          & 80.5          & 98.7          & 93.1          & 88.9    \\
        $\text{CERSA}_{\alpha=0.95,\beta=0.8}$           & \textbf{92.3} & 98.1          & 95.1          & 81.5          & \textbf{70.4} & 79.0          & 98.6          & \textbf{93.2} & 88.5    \\
        $\text{CERSA}_{\alpha=0.95,\beta=0.5}$           & 91.5          & 95.4          & 94.3          & 75.8          & 60.2          & 77.6          & 98.4          & \textbf{93.2} & 85.8    \\
        \bottomrule
    \end{tabular}
    }
    \vspace{-.1in}
    \caption{Evaluation results of CERSA on eight image classification datasets
    under different $\alpha$ and $\beta$ settings using ViT-Base~\citep{dosovitskiy2021image}.
    }
    \label{tab:different-cumulation-energy-retention-rate-full}
    \vspace{-.1in}
\end{table*}

In addition to testing the performance of CERSA on ViT-Large \citep{dosovitskiy2021image},
we also test it on ViT-Base \citep{dosovitskiy2021image}. With ViT-Base~\cref{tab:vit_base_results},
CERSA achieve an average accuracy of 89.0\% across eight datasets, outperforming
full parameter fine-tuning ({\bf FT}) (86.5\%) and significantly surpassing other
PEFT methods like LoRA~\citep{hu2022lora} (77.6\%), PiSSA~\citep{meng2024pissa} (84.2\%),
SVFT~\citep{lingam2024svft} (84.6\%), and SVFit~\citep{sun2024svfit} (83.7\%). It
also excels on fine-grained classification tasks, particularly Stanford Cars and
FGVC Aircraft~\citep{maji2013fine}, and matches or exceeds FT on general
datasets like CIFAR-10~\citep{krizhevsky2009learning}, Oxford Pets~\citep{parkhi2012cats},
and DTD~\citep{cimpoi2014describing}, demonstrating strong generalization. Compared
to ViT-Large~\citep{dosovitskiy2021image}, the performance advantage of our method
is more obvious on ViT-base~\citep{dosovitskiy2021image}, but the compression rate
is not as good as that of the large model.

Additionally, in \cref{tab:different-cumulation-energy-retention-rate-full}, we evaluate
the performance of all image classification tasks on ViT-Base~\citep{dosovitskiy2021image}
under different settings of the cumulative energy retention ratio. In the first set
of experiments, we set $\alpha=\beta$, which means that the entire principal subspace
corresponding to the cumulative energy is fine-tuned. We test performance under
different cumulative energy retention ratios \{0.95, 0.9, 0.8, 0.5\}. In the
second set of experiments, we fix $\alpha=0.95$ and examine the performance of fine-tuning only a portion of the principal subspace, with $\beta$ set to 0.9, 0.8, and 0.5,
respectively.

In the first set of experiments, we observe that the average performance drop is
minimal (only 1.4\%) when the cumulative energy retention ratio ranges from 0.95
to 0.8. Only when the ratio decreased to 0.5 did a significant performance
decline occur. This indicates that we have ample room to trade off a slight performance
loss for a substantial reduction in overall memory consumption. In the second
set of experiments, we found that reducing $\beta$ from 0.95 to 0.8 results in only
a performance drop of 0.5\%, and even at $\beta=0.5$, the performance decrease
is limited to 3.2\%. This suggests that fine-tuning only the most principal part
of the preserved subspace allows for a more parameter-efficient approach while
incurring only a minor performance loss.

\section{More Implementation Details}

\para{Experimental Environment.} All experiments were conducted on an NVIDIA L40
GPU using the PyTorch framework~\citep{paszke2019pytorch} and Hugging Face's \texttt{Transformers} library~\citep{wolf2020transformers} for fine-tuning.

\para{Settings for Image Classification.} For image classification, we fine-tune
ViT-Base and ViT-Large~\citep{dosovitskiy2021image} on the Query ({\bf Q}), Key ({\bf K}), and Value ({\bf V}) matrices within the attention module. In our method, CERSA, we
set a cumulative energy retention rate of $\alpha = \beta = 0.95$ across all fine-tuning
tasks. For comparison, we configure LoRA~\citep{hu2022lora} and PiSSA~\citep{meng2024pissa} with a rank of 32, a commonly chosen value that balances performance and the number of trainable parameters.

For SVFit~\citep{sun2024svfit}, we adhere to the recommended configuration of
the original paper, using a rank of 768 for all models. Similarly, for SVFT~\citep{lingam2024svft}, we adopt the best-performing settings. We use the AdamW optimizer~\citep{Loshchilov2017DecoupledWD} with a fixed batch size of 32 and a linear scheduler incorporating a warm-up ratio of 0.08. For further details on hyperparameter settings, see \cref{tab:img-cls-hyper-large} and \cref{tab:img-cls-hyper-base}.

\begin{table*}
    [!t]
    \centering
    \resizebox{\linewidth}{!}{%
    \begin{tabular}{lcccccccc}
        \toprule Dataset  & CIFAR-100 & EuroSAT & RESISC45 & StanfordCars & FGVC-Aircraft & DTD  & CIFAR-10 & OxfordPets \\
        \midrule           % Optimizer & \multicolumn{8}{c}{AdamW} \\
        % Warmup Ratio & \multicolumn{8}{c}{0.08} \\
        % LR Schedule & \multicolumn{8}{c}{Linear} \\
        % Epochs & \multicolumn{8}{c}{15} \\
        % Batch Size & \multicolumn{8}{c}{32} \\
        Attention Dropout & 0.1       & 0.1     & 0.1      & 0            & 0.1           & 0.1  & 0.1      & 0          \\
        Weight Decay      & 1e-3      & 1e-3    & 1e-3     & 0.01         & 1e-3          & 1e-3 & 1e-3     & 0.01       \\
        LR                & 1e-4      & 8e-5    & 1e-3     & 1e-3         & 2e-3          & 3e-4 & 1e-4     & 1e-4       \\
        LR (Classifier)   & 1e-3      & 5e-4    & 3e-3     & 3e-3         & 6e-3          & 1e-3 & 1e-3     & 1e-3       \\
        \bottomrule
    \end{tabular}}
    \vspace{-.1in}
    \caption{Hyperparameter settings for ViT-Large~\citep{dosovitskiy2021image} across
    different datasets for image classification experiments. LR: Learning Rate.}
    \label{tab:img-cls-hyper-large}
    \vspace{-.05in}
\end{table*}

\begin{table*}
    [!t]
    \centering
    \resizebox{\linewidth}{!}{%
    \begin{tabular}{lcccccccc}
        \toprule Dataset           & CIFAR-100 & EuroSAT & RESISC45 & StanfordCars & FGVC-Aircraft & DTD  & CIFAR-10 & OxfordPets \\
        \midrule Attention Dropout & 0.1       & 0.1     & 0.1      & 0            & 0.1           & 0.1  & 0.1      & 0          \\
        Weight Decay               & 1e-3      & 1e-3    & 1e-3     & 0.01         & 1e-3          & 1e-3 & 1e-3     & 0.01       \\
        LR                         & 2e-4      & 1e-4    & 2e-3     & 2e-3         & 1e-3          & 2e-4 & 2e-4     & 2e-4       \\
        LR (Classifier)            & 1e-3      & 5e-4    & 5e-3     & 5e-3         & 5e-3          & 1e-3 & 1e-3     & 1e-3       \\
        \bottomrule
    \end{tabular}}
    \vspace{-.1in}
    \caption{Hyperparameter settings for ViT-Base~\citep{dosovitskiy2021image} across
    different datasets for image classification experiments. LR: Learning Rate.}
    \label{tab:img-cls-hyper-base}
    \vspace{-.05in}
\end{table*}

\begin{table*}
    [!t]
    \centering
    \begin{tabular}{lcccccccc}
        \toprule Dataset   & MNLI & SST-2 & MRPC & CoLA & QNLI & QQP  & RTE  & STS-B \\
        \midrule            % Optimizer & \multicolumn{8}{c}{AdamW} \\
        % Warmup Ratio & \multicolumn{8}{c}{0.08} \\
        % LR Schedule & \multicolumn{8}{c}{Linear} \\
        Max Seq. Len.      & 256  & 128   & 320  & 64   & 512  & 320  & 320  & 128   \\
        Epochs             & 8    & 16    & 30   & 15   & 10   & 8    & 20   & 15    \\
        Batch Size         & 16   & 32    & 16   & 16   & 32   & 16   & 16   & 32    \\
        Classifier Dropout & 0.15 & 0     & 0    & 0.1  & 0.1  & 0.2  & 0.2  & 0.2   \\
        Weight Decay       & 0    & 0.01  & 0.01 & 0    & 0.01 & 0.01 & 0.01 & 0.1   \\
        LR                 & 1e-4 & 1e-4  & 2e-4 & 1e-4 & 1e-4 & 1e-4 & 2e-4 & 2e-4  \\
        LR(Classifier)     & 3e-4 & 3e-4  & 4e-4 & 3e-4 & 3e-4 & 3e-4 & 4e-4 & 4e-4  \\
        \bottomrule
    \end{tabular}
    \vspace{-.1in}
    \caption{Hyperparameter settings for DeBERTa-V3-Base~\citep{he2023debertav3}
    across different datasets for NLU experiments. LR: Learning Rate.}
    \vspace{-.10in}
    \label{tab:nlu-hyper}
\end{table*}

% \begin{table}[!h]
% \centering
% \setlength{\tabcolsep}{8mm}
% \resizebox{.9\linewidth}{!}{%
% \begin{tabular}{l|c|c}
% \hline
% \multicolumn{3}{c}{ViT-Base \citep{dosovitskiy2021image}} \\
% \hline
% Datasets & Classifier LR & Other LR \\
% \hline
% CIFAR-100 \citep{krizhevsky2009learning} & 2e-4 & 1e-4 \\
% EuroSAT \citep{helber2019eurosat} & 1e-4 & 5e-5 \\
% RESISC45 \citep{cheng2017remote} & 4e-3 & 2e-3 \\
% StanfordCars \citep{krause20133d} & 4e-3 & 2e-3 \\
% FGVC Aircraft \citep{maji2013fine} & 8e-3 & 4e-3 \\
% DTD \citep{cimpoi2014describing}   & 8e-4 & 4e-4 \\
% CIFAR-10 \citep{krizhevsky2009learning}  & 2e-4 & 1e-4 \\
% OxfordPets \citep{parkhi2012cats}  & 1e-4 & 5e-5 \\
% \hline
% \multicolumn{3}{c}{ViT-Large \citep{dosovitskiy2021image}} \\
% \hline
% Datasets & Classifier LR & Other LR \\
% \hline
% CIFAR-100 \citep{krizhevsky2009learning} & 1e-4 & 5e-5 \\
% EuroSAT \citep{helber2019eurosat}  & 5e-5 & 2e-5 \\
% RESISC45 \citep{cheng2017remote} & 1e-3 & 5e-4 \\
% StanfordCars \citep{krause20133d} & 4e-3 & 2e-3 \\
% FGVC Aircraft \citep{maji2013fine} & 5e-3 & 2e-3 \\
% DTD \citep{cimpoi2014describing}   & 6e-4 & 3e-4 \\
% CIFAR-10 \citep{krizhevsky2009learning}  & 1e-4 & 5e-5 \\
% OxfordPets \citep{parkhi2012cats}  & 6e-4 & 3e-4 \\
% \hline
% \end{tabular}
% }
% \vspace{-.1in}
% \caption{Learning rate settings for ViT-Base and ViT-Large across different datasets for image classification. LR: Learning Rate.}
% \label{tab:lr}
% \end{table}

\para{Settings for Text-to-Image Generation.} For the subject-driven text-to-image
generation task, We use Stable Diffusion v2-1-base~\citep{rombach2022high} as the pre-trained model and apply DreamBooth~\citep{ruiz2023dreambooth} for subject-driven text-to-image fine-tuning. We follow the setup of DreamBooth~\citep{ruiz2023dreambooth} to
evaluate CERSA’s fine-tuning. This ensures that the method captured subject-specific
details while preserving pre-trained knowledge. We compare CERSA with full-parameter
DreamBooth~\citep{ruiz2023dreambooth} and LoRA~\citep{hu2022lora}, evaluating image quality and textual alignment. 

In our implementation, we replace all linear layers in the UNet~\citep{ronneberger2015u}
and the attention modules of the CLIP~\citep{radford2021learning} text encoder with
CERSA. The cumulative energy retention rate is set to $\alpha=\beta=0.95$. For LoRA~\citep{hu2022lora},
we insert the adapters into the same layers with a rank of 32. For full-parameter
fine-tuning, we made all these layers trainable. The VAE(variational autoencoder)
module remains frozen in all methods. We use the AdamW~\citep{Loshchilov2017DecoupledWD}
optimizer and a constant scheduler. To ensure fairness in the inference stage,
we use identical random seeds, inference steps, and guidance scales across all
methods, preventing variations due to different parameter settings.

\para{Settings for NLU Experiments.} For the NLU experiments, we fine-tune the Q,
K, and V matrices in DeBERTa-v3-base~\citep{he2023debertav3}. The adapter rank for
LoRA~\citep{hu2022lora} and PiSSA~\citep{meng2024pissa} is set to 32. SVFit~\citep{sun2024svfit}
and SVFT~\citep{lingam2024svft} use the same settings as in the image classification
experiments. To ensure fairness, we follow SVFT's~\citep{lingam2024svft} max
sequence length settings. We use the AdamW~\citep{Loshchilov2017DecoupledWD}
optimizer and employ a linear scheduler with a warm-up ratio of 0.08. For detailed
hyperparameters, see \cref{tab:nlu-hyper}.

\begin{table*}
    [!t]
    \centering
    \resizebox{\linewidth}{!}{%
    \begin{tabular}{lcccccc}
        \toprule Methods                   & \makecell{Trainable\\ Parameter (M)} & \makecell{Trainable\\ Ratio (\%)} & \makecell{Weights\\ Memory (MB)} & \makecell{Optimizer\\ State Memory (MB)} & \makecell{Gradient\\ Memory (MB)} & \makecell{Total\\ Memory (MB)} \\
        \midrule FT                        & 303.3                                & 100                               & \textbf{1157.7}                  & 2314.4                                   & 1157.7                            & 4629.8                         \\
        \hline
        LoRA~\citep{hu2022lora}($r=8$)     & 0.8                                  & 0.3                               & 1161.9                           & 9.4                                      & 4.7                               & 1175.9                         \\
        LoRA~\citep{hu2022lora}($r=32$)    & 3.2                                  & 1.0                               & 1175.4                           & 36.4                                     & 18.2                              & 1229.9                         \\
        \hline
        SVFit~\citep{sun2024svfit}         & 0.04                                 & 0.02                              & 1349.8                           & 1.1                                      & 0.6                               & 1351.5                         \\
        SVFT~\citep{lingam2024svft}        & 0.12                                 & 0.06                              & 1350.9                           & 3.3                                      & 1.7                               & 1355.8                         \\
        \hline
        $\text{CERSA}_{\alpha=\beta=0.95}$ & 10.5                                 & 3.6                               & 1111.4                           & 80.5                                     & 40.3                              & 1232.2                         \\
        $\text{CERSA}_{\alpha=\beta=0.92}$ & 8.1                                  & 2.8                               & 1069.0                           & 59.0                                     & 29.5                              & \textbf{1157.5}                \\
        $\text{CERSA}_{\alpha=\beta=0.9}$  & 6.3                                  & 2.3                               & 1048.6                           & 49.3                                     & 24.6                              & \textbf{1122.5}                \\
        $\text{CERSA}_{\alpha=\beta=0.85}$ & 4.3                                  & 1.6                               & 1011.4                           & 33.3                                     & 16.7                              & \textbf{1061.4}                \\
        $\text{CERSA}_{\alpha=\beta=0.8}$  & 3.0                                  & 1.2                               & 985.2                            & 23.6                                     & 11.8                              & \textbf{1020.6}                \\
        \hline
    \end{tabular}}
    \vspace{-.1in}
    \caption{Memory consumption comparison across various methods with different
    settings. }
    \label{tab:memory_comparison}
    \vspace{-.15in}
\end{table*}

\section{Performance on Memory Consumption}
\cref{tab:memory_comparison} compares the parameter and memory efficiency of various
fine-tuning methods. We exclude activation and dataset-related memory usage, as they
remain largely independent of the fine-tuning approach. Thus, total memory refers
to the sum of the weight size, the gradient size, and the size of the optimizer parameter.
Besides, in \cref{tab:memory_comparison}, we report the number of trainable
parameters in millions (M).

Full-parameter fine-tuning (FT) updates all model parameters (303.3 M trainable parameters),
resulting in a substantial total memory consumption of 4629.8 MB. This high
memory demand makes FT impractical for resource-constrained environments.

LoRA~\citep{hu2022lora}, with ranks of 8 and 32, significantly reduces the number
of trainable parameters to 0.8 M and 3.2 M, respectively. However, its total
memory consumption remains considerable -- 1175.9 MB for rank=8 and 1229.9 MB for
rank=32 -- exceeding the memory footprint of the pre-trained weights due to
additional optimizer state and gradient storage. Similarly, SVFit~\citep{sun2024svfit}
achieves high parameter efficiency with only 0.04 M of trainable parameters yet still
requires 1351.5 MB of total memory, primarily due to the storage overhead of
full singular vector matrices.

The proposed CERSA method provides a flexible solution for memory and parameter-efficient fine-tuning by adjusting the cumulative energy retention rate, enabling different levels of efficiency based on memory constraints. For example, with a relatively ample memory budget, setting the retention rate to $\alpha=\beta=0.95$ yields
better performance. At $\alpha=\beta=0.92$, CERSA maintains a memory footprint
equivalent to the pre-trained weights during fine-tuning. When reduced to $\alpha
=\beta=0.8$, it retains 3.0 M trainable parameters comparable to LoRA~\citep{hu2022lora}
(rank=32) while significantly lowering total memory consumption to 1020.6 MB (while
LoRA~\citep{hu2022lora} uses up to 1229.9 MB).

Although not as parameter-efficient as SVFit~\citep{sun2024svfit} and SVFT~\citep{lingam2024svft},
CERSA excels in overall memory efficiency, even with more trainable parameters. This
makes it particularly advantageous for fine-tuning large-scale models in memory-constrained
environments. Additionally, its adjustable cumulative energy retention rate allows
for customized trade-offs, making CERSA a versatile solution that outperforms
other PEFT methods in total memory consumption while maintaining competitive
performance.

\section{Performance on Speed}
CERSA decomposes the pre-trained weight matrix into three components: $\bm{U}_{p}$,
$\bm{S}_{p}$, and $\bm{V}_{p}$. For simplicity, we assume that CERSA is configured
with $\alpha = \beta = 0.95$. Compared to the original weight matrix $\bm{W}$, this
decomposition introduces more granular matrix computations. However, since the size
of the matrices involved in computation is significantly reduced, the overall
computational cost is also reduced. To evaluate the actual impact on fine-tuning,
we design experiments to measure \textit{training throughput} and \textit{training
time}.

To eliminate the impact of dataset pre-processing and batch size on computation
time and throughput, we fix the batch size at 32 and the number of epochs at 15.
Fine-tuning is performed on ViT-Large~\citep{dosovitskiy2021image} across full-parameter
fine-tuning, LoRA~\citep{hu2022lora}, and CERSA.

\begin{figure}[!t]
    \centering
    \includegraphics[width=0.8\linewidth]{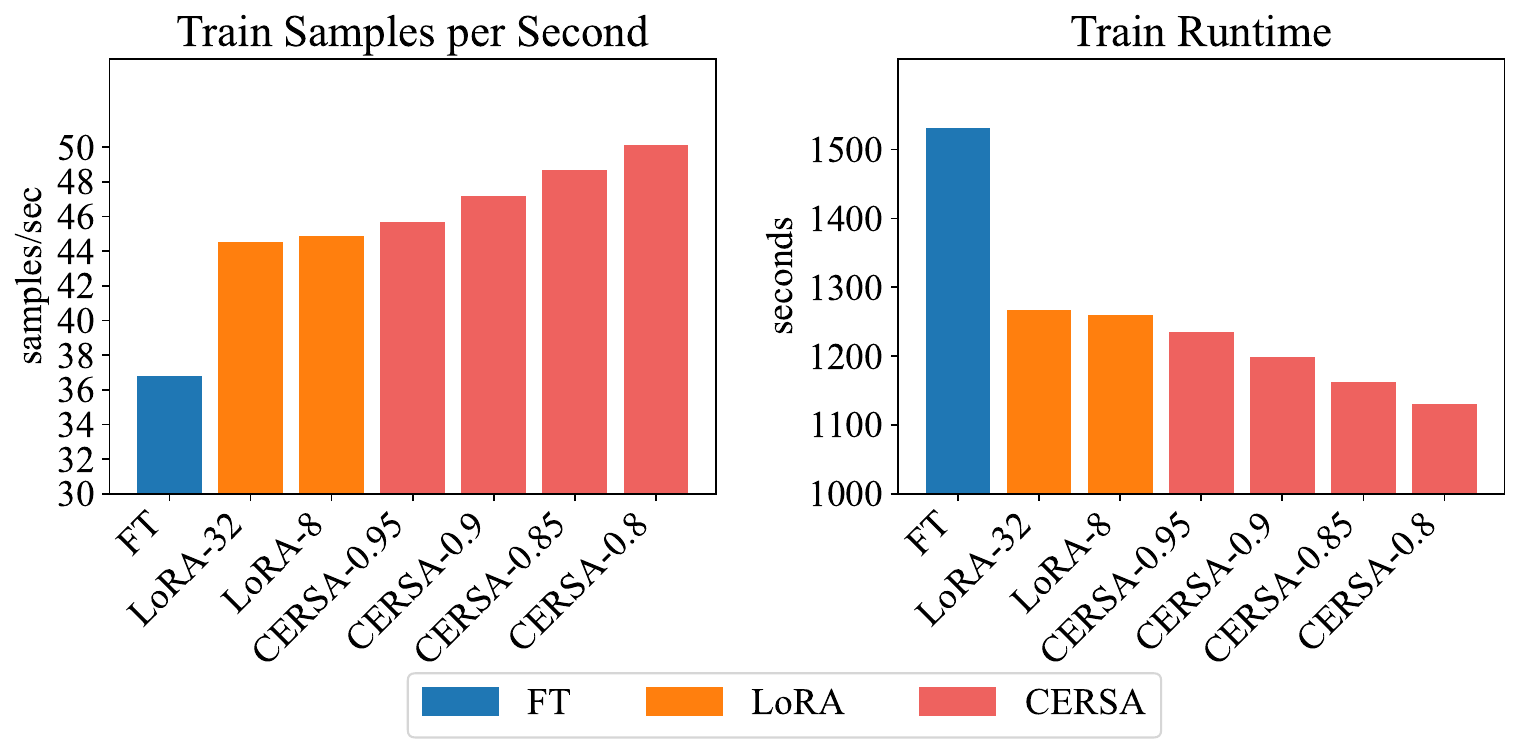}
    \vspace{-.10in}
    \caption{Training throughput and training time of fine-tuning ViT-Large~\citep{dosovitskiy2021image}
    on the DTD~\citep{cimpoi2014describing} dataset under various configurations.}
    \label{fig:throughoutput}
    \vspace{-.20in}
\end{figure}

Experimental results show that our method achieves a comparable or superior
training efficiency to LoRA while significantly outperforming FT in terms of
speed. As shown in \cref{fig:throughoutput}, LoRA ($r$=32, $r$=8) improves throughput
by about 30\% over FT. CERSA, across all cumulative energy retention rates \{0.95,
0.9, 0.85, 0.8\}, slightly exceeds LoRA’s efficiency, demonstrating that the cumulative
energy retention decomposition of the weight matrix effectively reduces computational
complexity while preserving model capacity. Despite introducing more granular
matrix multiplications, the significantly reduced dimensionality effectively lowers
the computational cost. As a result, CERSA matches or even surpasses LoRA in fine-tuning
speed.

\begin{figure*}[!t]
    \centering
    \includegraphics[width=\linewidth]{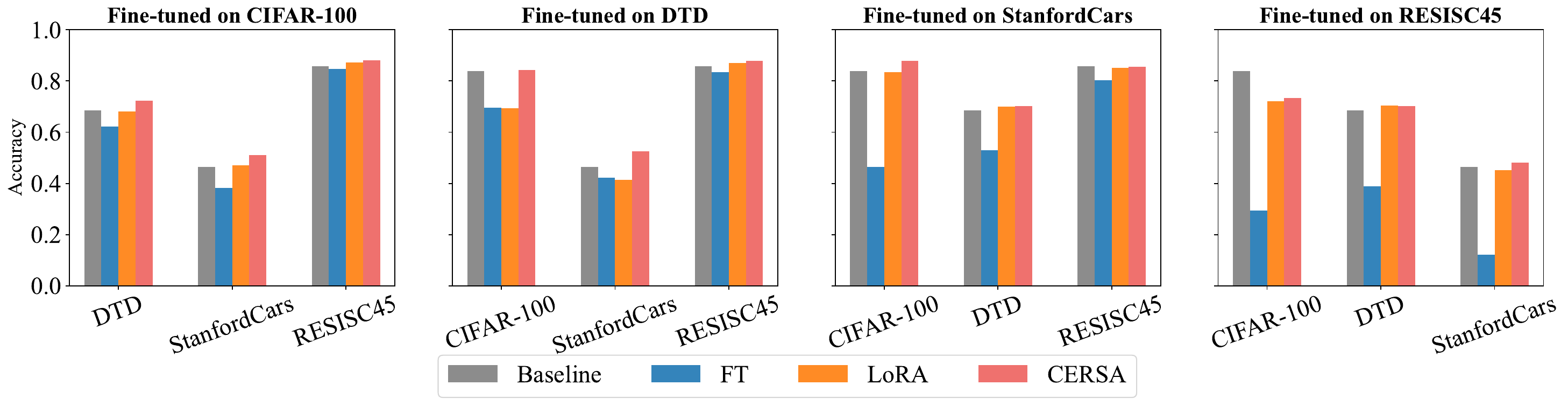}
    \vspace{-.1in}
    \caption{Out-of-distribution evaluation on various tasks.}
    \label{fig:forget}
    \vspace{-.2in}
\end{figure*}

\section{Performance on Out-of-Distribution Tasks}

During full-parameter fine-tuning, the model gradually forgets core features
from the pre-training data as its parameter space shifts significantly. In contrast,
CERSA restricts updates to $\bm{S}_{p}$, adjusting only the most critical feature
subspace while ensuring that the principal subspace remains unaffected by less important
dimensions. This preserves essential pre-trained knowledge.

Out-of-distribution(OOD) performance is a crucial indicator of knowledge retention,
as previously studied in \citep{hendrycks2017a} and \citep{kumar2022finetuning}.
\cref{fig:forget} shows the OOD performance of models fine-tuned on each of the four
datasets (CIFAR-100~\citep{krizhevsky2009learning}, DTD~\citep{cimpoi2014describing},
StanfordCars~\citep{krause20133d}, and RESISC45~\citep{cheng2017remote}), with accuracy
evaluated on the remaining three datasets. We compare FT, LoRA~\citep{hu2022lora},
and our proposed CERSA method. The gray bars indicate the model's original
performance before fine-tuning, serving as a reference for relative performance degradation.

Across all fine-tuning settings, CERSA consistently achieves superior OOD performance compared to FT and LoRA~\citep{hu2022lora}.
Specifically, when fine-tuned on CIFAR-100~\citep{krizhevsky2009learning} (leftmost
subplot), CERSA maintains a higher average OOD accuracy than LoRA~\citep{hu2022lora}
and FT, suggesting that it better preserves pre-trained knowledge for handling novel
tasks such as DTD~\citep{cimpoi2014describing} and StanfordCars~\citep{krause20133d},
or even leverages knowledge from CIFAR-100~\citep{krizhevsky2009learning}. A
similar trend is observed in the DTD~\citep{cimpoi2014describing} fine-tuning scenario
(second subplot), where CERSA demonstrates stronger retention of pre-trained
features, particularly on CIFAR-100~\citep{krizhevsky2009learning}.

\begin{table}[!t]
    \centering
    \resizebox{0.6\linewidth}{!}{%
    \begin{tabular}{lccc}
        \toprule Method                  & FT     & LoRA~\citep{hu2022lora} & CERSA  \\
        \midrule Average Forgetting Rate & 17.8\% & 2.3\%                   & -1.5\% \\
        \bottomrule
    \end{tabular}
    }
    \vspace{-.1in}
    \caption{Average forgetting rate of FT, LoRA~\citep{hu2022lora} and CERSA on
    the four datasets (CIFAR-100~\citep{krizhevsky2009learning}, DTD~\citep{cimpoi2014describing},
    StanfordCars~\citep{krause20133d}, and RESISC45~\citep{cheng2017remote})}
    \label{tab:average_forgetting_rate}
    \vspace{-.15in}
\end{table}

In our experiments, fine-tuning is performed on one dataset while accuracy is
evaluated on the remaining three. The average forgetting rate is defined as the ratio
of the average accuracy drop in the three out-of-distribution tasks compared to the
baseline accuracy of the pre-trained model after fine-tuning on a specific task.
As shown in \cref{tab:average_forgetting_rate}, these results highlight CERSA’s ability
to mitigate catastrophic forgetting by retaining key representations learned during
pre-training, thereby preserving higher accuracy on tasks not directly involved
in fine-tuning.

\begin{table*}
    [t]
    \centering
    \renewcommand{\arraystretch}{1.2}
    \begin{tabular}{c|cccc}
        \hline
        Dataset & CIFAR-100       & EuroSAT         & RESISC45        & StanfordCars    \\
        \hline
        Q       & 99.69\%/99.65\% & 99.94\%/99.94\% & 99.81\%/99.79\% & 99.95\%/99.94\% \\
        K       & 99.76\%/99.74\% & 99.96\%/99.96\% & 99.86\%/99.85\% & 99.94\%/99.94\% \\
        V       & 99.58\%/99.58\% & 99.91\%/99.91\% & 99.79\%/99.79\% & 99.92\%/99.92\% \\
        \hline
    \end{tabular}
    \vspace{0.5cm}
    \begin{tabular}{c|cccc}
        \hline
        Dataset & FGVC-Aircraft   & DTD             & CIFAR-10        & OxfordPets      \\
        \hline
        Q       & 99.76\%/99.73\% & 99.91\%/99.90\% & 99.68\%/99.63\% & 99.96\%/99.94\% \\
        K       & 99.78\%/99.76\% & 99.94\%/99.94\% & 99.74\%/99.72\% & 99.94\%/99.93\% \\
        V       & 99.69\%/99.70\% & 99.85\%/99.86\% & 99.55\%/99.55\% & 99.89\%/99.90\% \\
        \hline
    \end{tabular}
    \vspace{-.23in}
    \caption{Principal subspace similarity between the Q, K, and V matrices of
    ViT-Large~\citep{dosovitskiy2021image} pre-trained on ImageNet-21K~\citep{deng2009imagenet}
    and the fine-tuned weights on various downstream image classification tasks.
    }
    \vspace{-.15in}
    \label{tab:subspace_similarity}
\end{table*}

\section{Subspace Similarity Analysis}

% \begin{proof}[Proof of \cref{the:subspace}]
% Since $\text{Span}(\bm{U}) \subseteq \text{Span}(\bm{Q})$, there exists a matrix $\bm{C} \in \mathbb{R}^{k \times k_d}$ such that
% $\bm{U} = \bm{Q} \bm{C}$.
% Similarly, there exists a matrix $\bm{D} \in \mathbb{R}^{k \times k_d}$ such that: $ \bm{V} = \bm{Q}'\bm{D} $.
% Substituting the above equations into the SVD decomposition of $\Delta \bm{W}$, we have
% $
%     \bm{M} = (\bm{Q} \bm{C}) \bm{\Sigma} (\bm{D}^T \bm{Q}'^T)
% $.
% Let $\bm{S} = \bm{C} \bm{\Sigma} \bm{D}^T$, then we have
% $\bm{M} = \bm{Q} \bm{S} \bm{Q}'^T$.

\begin{figure*}[!t]
    \centering
    \includegraphics[width=1\linewidth]{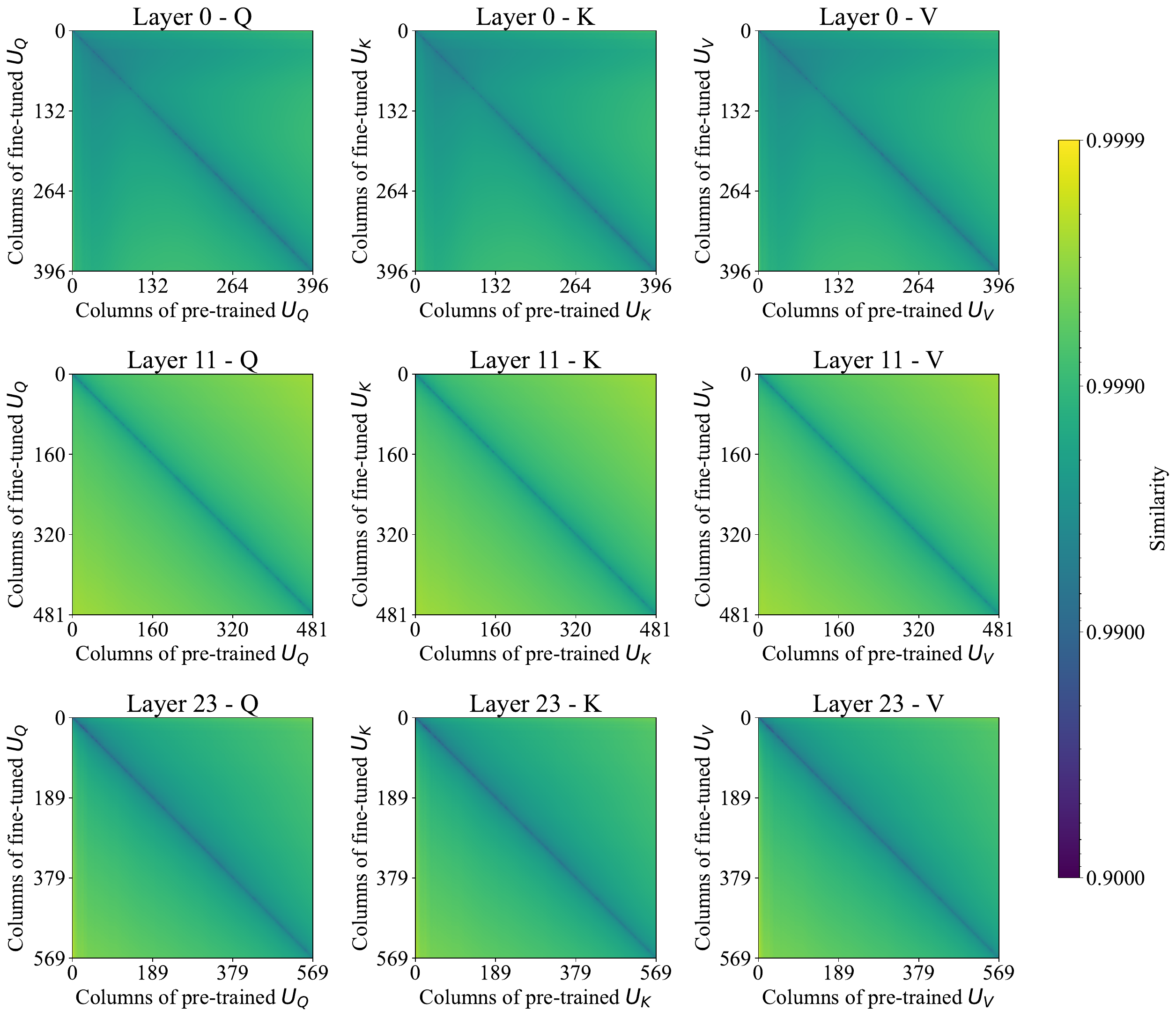}
    \vspace{-.2in}
    \caption{The similarity between the principal output subspace $\bm{U}_{p}$
    of the pre-trained and fine-tuned weights for the Q, K, and V matrices in
    layers 0, 11, and 23 of ViT-Large~\citep{dosovitskiy2021image}. The x-axis
    represents the subspace spanned by the top-$i$ singular vectors of the pre-trained
    weights, while the y-axis represents the subspace spanned by the top-$j$ singular
    vectors of the fine-tuned weights. }
    \label{fig:heatmap-U}
    \vspace{-.1in}
\end{figure*}

\begin{figure*}[!t]
    \centering
    \includegraphics[width=1\linewidth]{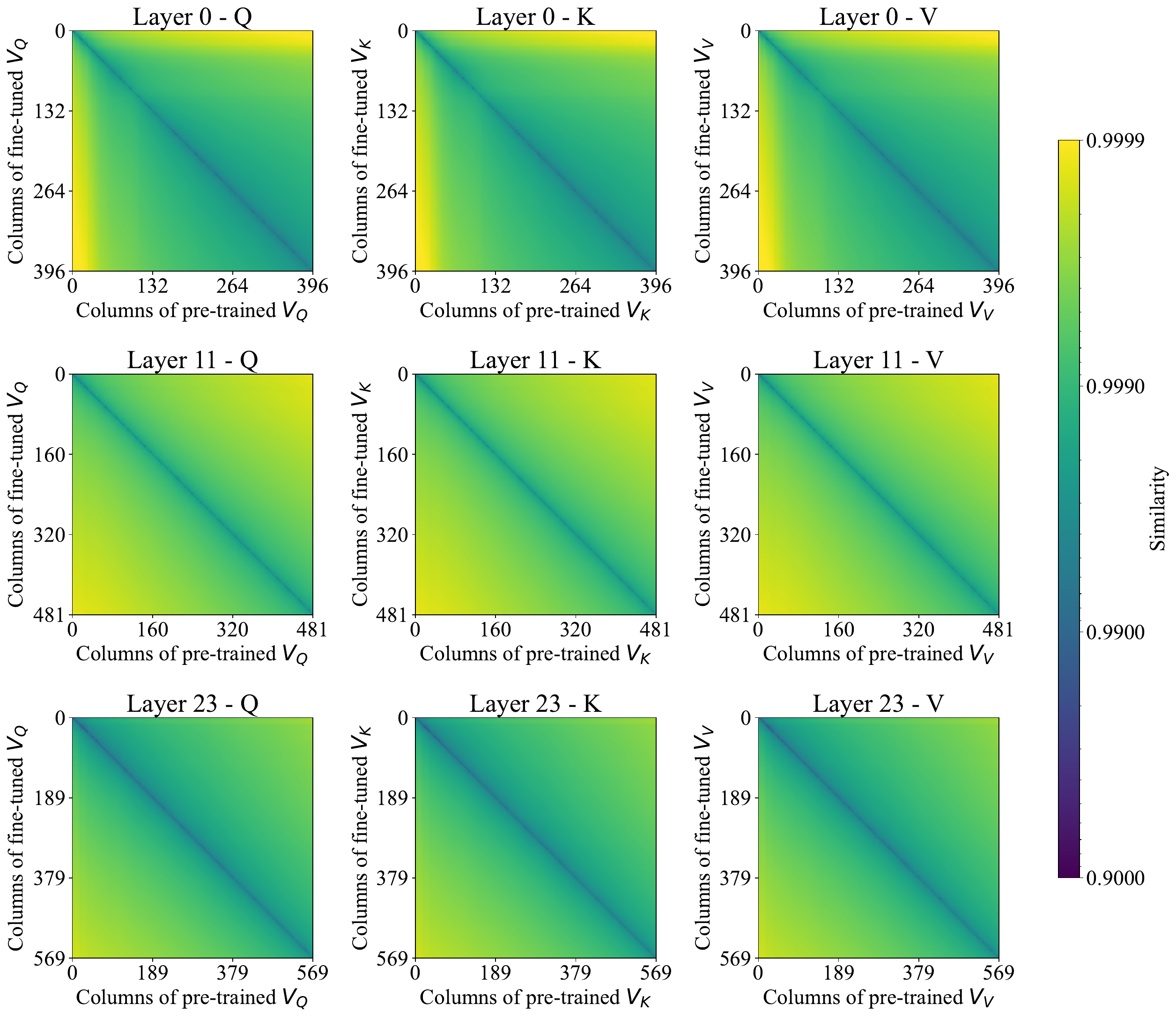}
    \vspace{-.2in}
    \caption{The similarity between the principal input subspace $\bm{V}_{p}$ of
    the pre-trained and fine-tuned weights for the Q, K, and V matrices in
    layers 0, 11, and 23 of ViT-Large~\citep{dosovitskiy2021image}. The x-axis
    represents the subspace spanned by the top-$i$ singular vectors of the pre-trained
    weights, while the y-axis represents the subspace spanned by the top-$j$ singular
    vectors of the fine-tuned weights.}
    \vspace{-.1in}
    \label{fig:heatmap-V}
\end{figure*}

\subsection{Subspace Similarity Between Pre-trained and Fine-tuned Models}
Our theoretical analysis assumes that CERSA can approximate full-parameter fine-tuning
based on the premise that the principal subspace of $\bm{W}^{\prime}$ after full
fine-tuning on a downstream task remains highly similar to that of the pre-trained
weights. This assumption suggests that fine-tuning primarily refines the
existing subspace rather than significantly altering its structure. To empirically
validate this assumption, we measure the subspace similarity between
$\bm{W}^{\prime}$ and the pre-trained weights $\bm{W}$ of the ViT-Large~\citep{dosovitskiy2021image}
model, initially trained on ImageNet-21K~\citep{deng2009imagenet}, across eight
different downstream image classification datasets.

To quantitatively assess this similarity, we employ the \textit{Grassmann
subspace similarity}~\citep{hu2022lora}, a metric that effectively captures the
alignment between the principal output subspaces of the pre-trained and fine-tuned
weights. Formally, the Grassmann similarity is defined as follows:

\begin{equation}
    \label{eq:grassmann}\psi(\bm{U}_{A}^{i}, \bm{U}_{B}^{j}) = \frac{\|\bm{U}_{A}^{i\top}\bm{U}_{B}^{j}\|_{F}^{2}}{\min\{i,
    j\}}, \quad \psi \in [0,1]
\end{equation}
where $\bm{U}_{A}^{i}$ and $\bm{U}_{B}^{j}$ stands for top-$i$ columns of the
$\bm{U}$ matrix from SVD decomposition of matrix $\bm{A}$ and top-$j$ columns of
the $\bm{U}$ matrix from SVD decomposition of matrix $\bm{B}$ respectively. $\psi$
ranges from 0 (completely disjoint subspaces) to 1 (identical subspaces).

Similarly, we extend this analysis to the input subspace represented by $\bm{V}$,
applying the same similarity computation. The results presented in
\cref{tab:subspace_similarity} are computed using a subspace that retains 95\% of
the cumulative energy. The first value represents the similarity of the output
subspace $\bm{U}$, while the second corresponds to the input subspace $\bm{V}$.

Analyzing \cref{tab:subspace_similarity}, we observe that across all downstream
tasks, the Grassmann subspace similarity between the fine-tuned and pre-trained subspaces
consistently exceeds 99.5\% for both $\bm{U}$ and $\bm{V}$ across all three
attention matrices -- Q, K, and V. This strong evidence suggests that fine-tuning
minimally affects the principal subspace of the pre-trained weights, thereby validating
our assumption.

To further examine the stability of the Grassmann subspace similarity under varying
top-$k$ selections, we conducted experiments on the 0th, 11th, and 23rd layers
of the ViT-Large~\citep{dosovitskiy2021image} model before and after fine-tuning
on CIFAR-10~\citep{krizhevsky2009learning}. Specifically, we extract the Q, K,
and V matrices from these layers, perform SVD to obtain the principal subspaces $\bm
{U}$ and $\bm{V}$ for both the pre-trained and fine-tuned models, and measure
subspace similarity by selecting top-$i$ and top-$j$ singular vectors. The results
are visualized in a heat map.

Since all similarity values fall within the range of 0.9 to 0.9999, we applied a
logarithmic transformation to the color scale for better visualization. As depicted
in \cref{fig:heatmap-U} and \cref{fig:heatmap-V}, the subspaces of the fine-tuned
and pre-trained weight matrices exhibit consistently high similarity across all
choices of top-$i$ and top-$j$, with values ranging from 0.9 to 0.9999. This
confirms that the observed subspace similarity is not confined to specific top-$k$
selections but persists across all choices of singular vectors.

Regardless of the truncation level, the fine-tuned and pre-trained weight matrices
maintain exceptionally high Grassmann subspace similarity. This finding further
substantiates our hypothesis that fine-tuning does not significantly alter the
principal subspace of the pre-trained model, reinforcing the fundamental assumption
underlying our method.

\subsection{Proofs}

\begin{proof}
    Let $\bm{M}\in \mathbb{R}^{m\times n}$ be a matrix of rank $k$. By the
    singular value decomposition (SVD), we can write $\bm{M}= \bm{U}\bm{\Sigma}\bm
    {V}^{T},$ where $\bm{U}\in \mathbb{R}^{m \times k}$ and $\bm{V}\in \mathbb{R}
    ^{n \times k}$ are matrices with orthonormal columns, and
    $\bm{\Sigma}\in \mathbb{R}^{k \times k}$ is a diagonal matrix with positive
    diagonal entries (the singular values).

    Since there exists an orthonormal basis
    $\bm{Q}= \{\bm{e}_{1}, \bm{e}_{2}, \dots, \bm{e}_{k}\}$ such that
    $\operatorname{Span}(\bm{U}) = \operatorname{Span}(\bm{Q}),$ both $\bm{U}$
    and $\bm{Q}$ form orthonormal bases for the same $k$-dimensional subspace. Therefore,
    there exists an orthogonal matrix $\bm{R}\in \mathbb{R}^{k \times k}$ such
    that
    \[
        \bm{U}= \bm{Q}\bm{R}.
    \]
    Similarly, because $\text{Span}(\bm{V}) = \text{Span}(\bm{Q'})$, there
    exists an orthogonal matrix $\bm{R}'\in \mathbb{R}^{k\times k}$ satisfying
    \[
        \bm{V}= \bm{Q'}\bm{R}'.
    \]
    Substituting these expressions into the singular value decomposition of
    $\bm{M}$, we obtain
    \[
        \bm{M}= \bm{U}\bm{\Sigma}\bm{V}^{T}= (\bm{Q}\bm{R})\bm{\Sigma}(\bm{Q'}\bm
        {R}')^{T}= \bm{Q}\bm{R}\bm{\Sigma}\bm{R}'^{T}\bm{Q'}^{T}.
    \]
    Defining $\bm{S}= \bm{R}\bm{\Sigma}\bm{R}'^{T}$, we have
    \[
        \bm{M}= \bm{Q}\bm{S}\bm{Q'}^{T},
    \]
    where $\bm{S}\in \mathbb{R}^{k\times k}$. This completes the proof.
\end{proof}

    % \clearpage

    % \bibliography{iclr2026_conference}
    % \bibliographystyle{iclr2026_conference}
\end{document}